\documentclass{article} 
\usepackage{iclr2027_conference,times}

\usepackage{amsmath,amsfonts,bm}

\def\eqref#1{equation~\ref{#1}}

\def\1{\bm{1}}

\DeclareMathAlphabet{\mathsfit}{\encodingdefault}{\sfdefault}{m}{sl}
\SetMathAlphabet{\mathsfit}{bold}{\encodingdefault}{\sfdefault}{bx}{n}

\usepackage[utf8]{inputenc} 
\usepackage[T1]{fontenc}    
\usepackage{hyperref}       
\usepackage{url}            
\usepackage{booktabs}       
\usepackage{amsfonts}       
\usepackage{nicefrac}       
\usepackage{microtype}      
\usepackage{xcolor}         
\usepackage{listings}

\usepackage{graphicx}
\usepackage{subcaption}
\usepackage{wrapfig}
\usepackage{booktabs,arydshln} 
\usepackage{bbding,multirow,colortbl}
\usepackage{comment}
\usepackage{enumitem,makecell}

\usepackage{algorithm,algpseudocode}

\newcommand{\Rmnum}[1]{\expandafter\@slowromancap\romannumeral #1@}

\usepackage{amsmath}
\usepackage{amssymb}
\usepackage{mathtools}
\usepackage{amsthm}

\usepackage{booktabs}
\usepackage{tabularx}
\usepackage{array}
\usepackage{makecell}
\usepackage{arydshln}
\usepackage{pifont}
\usepackage{xspace}

\usepackage{booktabs}
\usepackage{array}
\usepackage{pifont}
\usepackage{xcolor}
\usepackage{colortbl}
\usepackage{tikz}
\usepackage{graphicx}

\definecolor{lightpurple}{RGB}{239,234,247}
\definecolor{trajcolor}{RGB}{70,110,200}
\definecolor{stepcolor}{RGB}{40,150,90}
\definecolor{actioncolor}{RGB}{225,135,50}

\newcommand{\xmark}{\ding{55}}

\newtheorem{theorem}{Theorem}

\newtheorem{proposition}{Proposition}

\newcolumntype{Y}{>{\raggedright\arraybackslash}X}
\newcolumntype{P}[1]{>{\raggedright\arraybackslash}p{#1}}

\usepackage{etoc}

\usepackage{hyperref}
\usepackage{url}
\definecolor{cite_color}{RGB}{0,0,100}
\definecolor{link_color}{RGB}{0,0,100}
\definecolor{url_color}{RGB}{0,0,120}
\definecolor{emp_color}{RGB}{0,0,255}
\hypersetup{
 colorlinks,
 citecolor=cite_color,
 linkcolor=link_color,
 urlcolor=url_color
}

\newcommand{\ourmethod}{{\fontfamily{lmtt}\selectfont \textbf{SCALE}}\xspace}

\definecolor{tblheader}{gray}{0.95} 
\definecolor{symblue}{RGB}{25, 100, 180} 
\definecolor{symred}{RGB}{180, 50, 50}   
\definecolor{peach}{RGB}{255,218,185} 

\definecolor{bestbg}{RGB}{246,232,183}    
\definecolor{secondbg}{RGB}{226,216,238}  
\definecolor{headerbg}{RGB}{112,91,145}    
\definecolor{subheaderbg}{RGB}{235,229,242} 
\definecolor{subheadertext}{RGB}{0,0,0}   
\definecolor{sectionbg}{RGB}{242,242,242}
\definecolor{oursbg}{RGB}{238,244,253}

\usepackage[most]{tcolorbox}
\usepackage{xcolor}
\usepackage{varwidth}

\definecolor{lalapurple}{HTML}{7656A8}
\definecolor{lalapurpledeep}{HTML}{59417F}
\definecolor{lalagold}{HTML}{E1AD3D}
\definecolor{lalacream}{HTML}{FFF9ED}
\definecolor{lalatext}{HTML}{2F2933}
\definecolor{promptvar}{HTML}{B56A28}
\definecolor{prompttag}{HTML}{6F45A5}

\newtcolorbox{promptbox}[1][]{
    enhanced,
    breakable,
    colback=lalacream,
    colframe=lalapurple,
    colbacktitle=lalapurpledeep,
    coltitle=white,
    fonttitle=\bfseries,
    fontupper=\small,
    title=#1,
    boxrule=1.15pt,
    arc=3.5mm,
    outer arc=3.5mm,
    left=10pt,
    right=10pt,
    top=9pt,
    bottom=9pt,
    before skip=8pt,
    after skip=9pt,
    borderline north={1.6pt}{0pt}{lalagold},
}

\definecolor{caseenv}{HTML}{C96A58}
\definecolor{caseenvdeep}{HTML}{A94F43}
\definecolor{caseenvbg}{HTML}{FFF5F0}

\definecolor{caseagent}{HTML}{4E6F9E}
\definecolor{caseagentdeep}{HTML}{39577F}
\definecolor{caseagentbg}{HTML}{F3F7FC}

\definecolor{caseaction}{HTML}{355C8A}
\definecolor{casethink}{HTML}{355C8A}
\definecolor{caseoutput}{HTML}{A14F42}

\usepackage{tcolorbox}
\tcbuselibrary{breakable,skins}
\usepackage{changepage}

\newtcolorbox{caseenvbox}[1]{
    enhanced,
    breakable,
    colback=caseenvbg,
    colframe=caseenv,
    colbacktitle=caseenvdeep,
    coltitle=white,
    fonttitle=\bfseries\rmfamily,
    fontupper=\small\rmfamily,
    title={Environment (#1)},
    boxrule=1.05pt,
    arc=3.5mm,
    outer arc=3.5mm,
    left=10pt,
    right=10pt,
    top=8pt,
    bottom=8pt,
    before skip=8pt,
    after skip=8pt
}

\newtcolorbox{caseagentbox}[1]{
    enhanced,
    breakable,
    colback=caseagentbg,
    colframe=caseagent,
    colbacktitle=caseagentdeep,
    coltitle=white,
    fonttitle=\bfseries\rmfamily,
    fontupper=\small\rmfamily,
    title={Agent (#1)},
    boxrule=1.05pt,
    arc=3.5mm,
    outer arc=3.5mm,
    left=10pt,
    right=10pt,
    top=8pt,
    bottom=8pt,
    before skip=8pt,
    after skip=8pt
}

\newenvironment{envstep}[1]
{%
    \par
    \begin{adjustwidth}{0pt}{0.16\linewidth}
    \begin{caseenvbox}{#1}
}
{%
    \end{caseenvbox}
    \end{adjustwidth}
    \par
}

\newenvironment{agentstep}[1]
{%
    \par
    \begin{adjustwidth}{0.16\linewidth}{0pt}
    \begin{caseagentbox}{#1}
}
{%
    \end{caseagentbox}
    \end{adjustwidth}
    \par
}

\newcommand{\casethink}[1]{%
    {\color{casethink}\ttfamily <think>}%
    \,\textnormal{#1}\,%
    {\color{casethink}\ttfamily </think>}%
}

\newcommand{\caseaction}[1]{%
    {\color{caseaction}\ttfamily <action>}%
    \,\textnormal{#1}\,%
    {\color{caseaction}\ttfamily </action>}%
}

\newcommand{\toolopen}{%
    \textcolor{caseaction}{\texttt{<tool\_call>}}%
}

\newcommand{\toolclose}{%
    \textcolor{caseaction}{\texttt{</tool\_call>}}%
}

\newcommand{\best}[1]{\cellcolor{bestbg}\textbf{#1}}
\newcommand{\second}[1]{\cellcolor{secondbg}\underline{#1}}

\newcolumntype{L}[1]{>{\raggedright\arraybackslash}p{#1}}
\newcolumntype{C}[1]{>{\centering\arraybackslash}p{#1}}

\title{From Outcomes to Steps: Reliable Success-Aware Critic Learning for Long-Horizon Language Agents}

\author{
  Jiaxi Liu\textsuperscript{1}, 
  Shixiong Kai\textsuperscript{2},
  Zhentao Tang\textsuperscript{2},
  Yang Li\textsuperscript{1}, 
  Yuqi Cui\textsuperscript{2},
  Mingxuan Yuan\textsuperscript{2},
  Junchi Yan\textsuperscript{1,*}\\
  \normalfont\textsuperscript{1} Shanghai Jiao Tong University \\
  \normalfont\textsuperscript{2} Huawei Noah’s Ark Lab
}

\iclrfinalcopy 
\begin{document}

\maketitle
\lhead{Preprint.}

\begingroup
\renewcommand{\thefootnote}{\fnsymbol{footnote}}
\footnotetext[1]{Corresponding to Junchi Yan (\texttt{yanjunchi@sjtu.edu.cn}).}
\endgroup

\begin{abstract}

Reinforcement learning for long-horizon language agents requires identifying whether intermediate actions make eventual task success more or less likely when supervision is available only at trajectory level.
The main difficulty is that \textit{a critic with aggregate return-prediction accuracy may still fail to resolve the local value differences needed to assign credit to actions}, especially under sparse feedback and function approximation. Conventional critics estimate future return, but policy improvement also depends on knowing whether each intermediate decision makes task success more or less likely.
We introduce \ourmethod, an actor--critic framework that augments conventional return estimation with explicit prediction of eventual task success. At each interaction prefix, the critic predicts a return value and a \textbf{Success Value}, defined as the probability that the task will eventually be completed from that prefix. 
\ourmethod converts changes in Success Value into multi-step action credit through \textbf{Success-GAE} and modulates its contribution according to the estimated reliability of the predictor, while retaining the token-level PPO objective.
Across five agent benchmarks and three model backbones, \ourmethod improves task performance over the compared RL baselines in most settings, with gains of up to \textbf{4.1} percentage points in success rate. Separate critic-learning analyses show that \ourmethod also improves critic estimation and training stability.

\end{abstract}

\vspace{-4pt}
\section{Introduction}
\label{sec-introduction}
\vspace{-4pt}

\begin{figure*}[t]
    \centering
    \includegraphics[width=\textwidth]{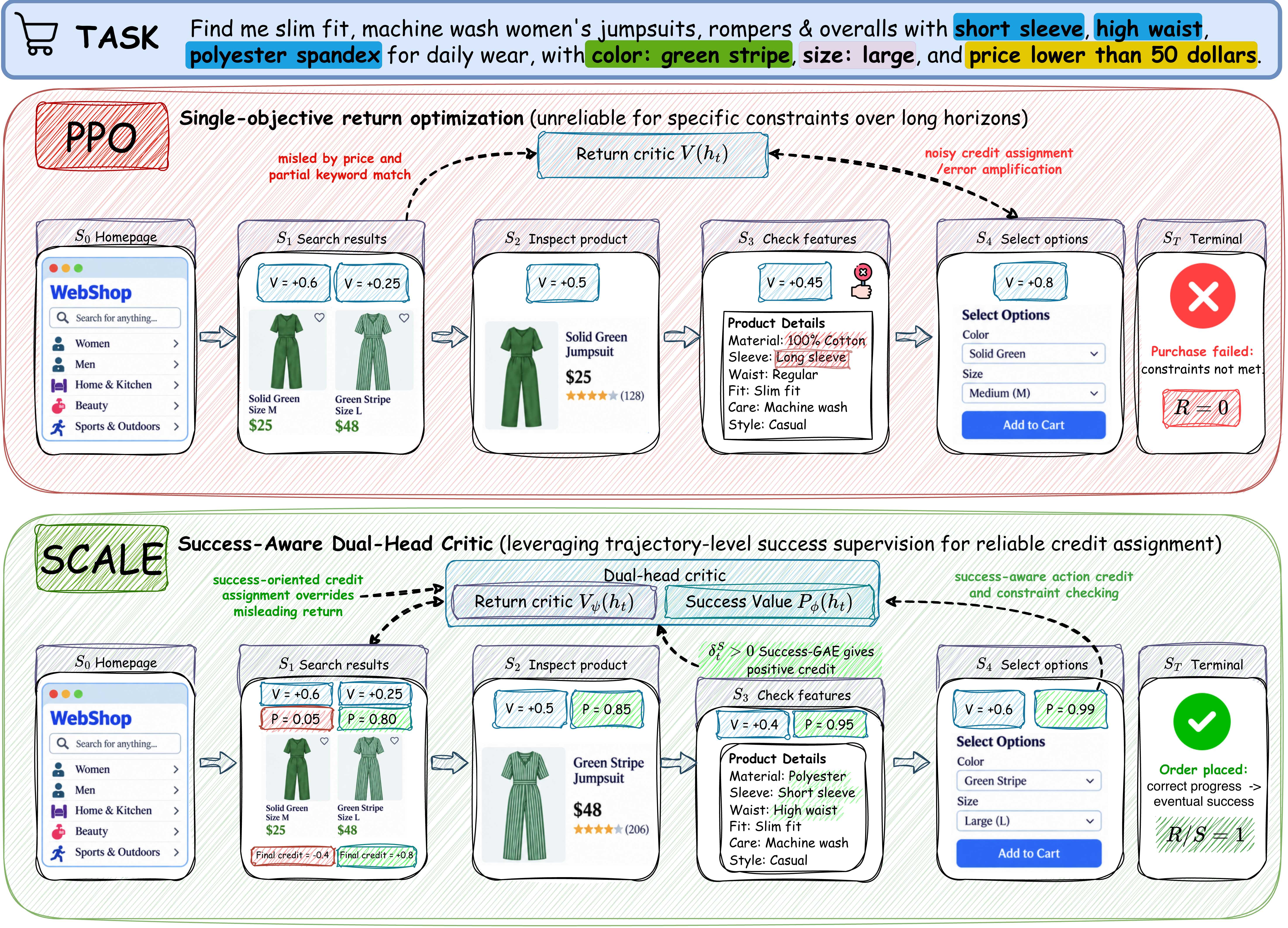}
    \caption{
    Comparison between PPO and \ourmethod. PPO derives step credit from changes in predicted return, which can be unreliable when feedback is sparse. An action may receive positive credit even when it lowers the probability of eventual success. \ourmethod instead uses changes in Success Value to measure whether each action makes task success more or less likely.
    }
    \vspace{-15pt}
    \label{fig:motivation}
\end{figure*}

Large language models (LLMs) are deployed as interactive agents that browse the web, use tools, and interact with environments over many steps \citep{yao2023react,shridhar2021alfworld,yao2022webshop}. Reinforcement learning (RL) provides a framework for improving such agents from task outcomes \citep{jin2025searchr1,feng2026retool,wang2025ragen,wang2026ragen2}, yet it remains difficult to tell which intermediate actions help or hurt the outcome. As illustrated in Fig.~\ref{fig:motivation}, an action can satisfy a local constraint yet make the task less likely to succeed. The learner therefore needs to determine which intermediate decisions improve the prospect of eventual task completion. For long-horizon language agents, this is difficult because supervision often arrives only after the entire interaction, while policy improvement depends on distinguishing the effects of individual decisions along the way.



Existing methods assign credit at different parts of the trajectory. Standard PPO and outcome-based methods optimize token probabilities from trajectory-level returns \citep{schulman2017ppo,ouyang2022instructgpt,shao2024deepseekmath}, while recent methods assign credit to individual turns or environment actions \citep{zeng2025multiturn,feng2025gigpo,li2025turnppo,wang2026capo,ding2026gtpo}. Critic-based methods further estimate intermediate values to propagate delayed outcomes \citep{zhou2024archer,kazemnejad2025vineppo,hou2026sao}. These methods assign credit at finer parts of the trajectory, but accurate prediction of future return does not necessarily imply reliable action-level credit. Under sparse supervision and function approximation, a change in predicted return may not tell us whether a particular action made eventual success more likely. The central issue is therefore not only where credit is assigned, but \textit{whether the critic preserves the local ordering between alternative actions} and \textit{whether its reliability can be assessed during training}. In Sec.~\ref{sec:preliminary_analysis}, we show that this concern is practical rather than hypothetical, conventional PPO critics produce action-level credit that degrades toward random chance at long horizons and remains unstable throughout training, leaving reliable critic learning for long-horizon credit assignment an open challenge.


To address these challenges, we introduce \ourmethod (\textbf{S}uccess-\textbf{C}alibrated \textbf{A}dvantage \textbf{L}earning and \textbf{E}stimation), a success-aware actor--critic framework for long-horizon language agents. \ourmethod makes eventual task completion an explicit prediction target alongside conventional return estimation, \textbf{rather than relying on return alone}. Specifically, the critic retains a return head for expected cumulative reward and adds a Success Value head that estimates the probability of satisfying the benchmark's terminal completion criterion from the current interaction prefix. It then uses changes in this prediction to assign credit to individual environment steps, reflecting whether each action improves eventual success. Because the reliability of this signal can vary during training, \ourmethod adjusts its contribution according to critic quality. The resulting credit complements the standard return-based advantage while retaining the token-level PPO objective. Sec.~\ref{sec-method} details the framework.

We evaluate \ourmethod on five representative agentic benchmarks spanning text-based interaction and tool use: ALFWorld \citep{shridhar2021alfworld}, WebShop \citep{yao2022webshop}, ScienceWorld \citep{wang2022scienceworld}, BFCL-v4 \citep{patil2025bfcl}, and AppWorld \citep{trivedi2024appworld}. Across these tasks, \ourmethod consistently \emph{outperforms} most strong agentic RL baselines and delivers substantial gains in task performance. Beyond end-task performance, We also analyze critic quality and the accuracy of its action credit to understand where these gains come from. Our contributions are threefold:

\vspace{-8pt}
\begin{itemize}[leftmargin=1.2em,itemsep=-0.2em]

\item[\ding{182}] \textbf{Local-credit failure of return critics.}
We show that conventional PPO critics become less reliable at distinguishing alternative actions as decisions move farther from terminal supervision.

\item[\ding{183}] \textbf{Success-Based Credit with Reliability Control.}
We introduce \ourmethod, which explicitly predicts eventual task success from \textbf{each interaction prefix} and \textbf{converts changes in this probability into multi-step action credit}. It then modulates the contribution of this credit according to \textbf{estimated critic reliability}, while retaining the standard token-level PPO objective.

\item[\ding{184}] \textbf{Empirical validation.}
Experiments across five agent benchmarks and three model backbones show improvements in task performance, including gains of up to \textbf{4.1} percentage points in success rate. Additional analyses show improved action ranking and support the roles of success-based credit and reliability control.

\end{itemize}
\vspace{4pt}

\vspace{-8pt}
\section{Related Work}
\label{sec-relate}
\vspace{-8pt}

\paragraph{Reinforcement Learning for LLM-based Agents.}
Reinforcement learning has evolved from preference alignment \citep{schulman2017ppo,ouyang2022instructgpt} toward optimization from verifiable task outcomes \citep{shao2024deepseekmath}, and has recently been extended to interactive language agents. Existing work studies RL for search and tool use \citep{jin2025searchr1,feng2026retool,qian2025toolrl}, multi-turn environment interaction \citep{wang2025ragen,wang2026ragen2}, and general agent-training frameworks \citep{feng2024agile,luo2025agentlightning}. Scalable RL infrastructures such as verl \citep{sheng2025hybridflow} and ROLL \citep{wang2025roll} support large-scale rollout generation and environment interaction. Other approaches improve long-horizon learning through planning and exploration \citep{paglieri2025learning,wang2026beacon,wang2026agentg2}. While these methods improve the training and exploration of language agents, a central challenge remains: \textit{how to convert sparse trajectory-level outcomes into reliable credit for intermediate actions}, particularly for decisions far from terminal supervision.

\vspace{-8pt}

\paragraph{Credit Assignment and Critic Learning.}
Existing approaches differ along two dimensions that are central to credit assignment: \emph{where} credit is assigned along a trajectory and \emph{how} that credit is estimated. Critic-free methods derive turn- or step-level signals through comparison, grouping, or structural information \citep{zeng2025multiturn,feng2025gigpo,he2026hgpo,wang2026beacon}, but often require repeated states, milestones, or other task-dependent structure. Monte Carlo methods estimate local action values from multiple continuations, avoiding a learned critic at the cost of substantial sampling overhead \citep{kazemnejad2025vineppo}. Critic-based methods instead propagate delayed outcomes through learned turn- or action-level values \citep{zhou2024archer,li2025turnppo,wang2026capo,tan2026hcapo}, but reliable local comparisons remain difficult when value estimates are learned from sparse supervision. Recent work further improves robustness by suppressing unreliable updates or controlling the use of uncertain credit signals \citep{li2026sorl,li2026ecpo,tao2026trace}. In contrast, \ourmethod explicitly predicts eventual task success from each interaction prefix and uses the estimated reliability of this prediction to control its contribution to environment-step credit.

\begin{table*}[t]
\centering
\small
\setlength{\tabcolsep}{5pt}
\renewcommand{\arraystretch}{1.1}

\caption{
Comparison of representative RL methods for credit assignment. ``Gran.'' denotes credit-assignment granularity. ``Critic'' indicates whether a learned critic is used. ``Reliab.'' denotes explicit reliability modeling or control, and ``Token PO'' indicates token-level policy optimization.
}
\label{tab:method_taxonomy}

\begin{tabular}{l c c c c l l}
\toprule
\textbf{Method}
& \textbf{Gran.} & \textbf{Critic} & \textbf{Reliab.} & \textbf{Token PO} & \textbf{Credit Signal} & \textbf{Extra Structure} \\
\midrule

PPO
& Step & \checkmark & \xmark & \checkmark & Discounted return & None \\

GRPO
& Trajectory & \xmark & \xmark & \checkmark & Group outcome & Group rollouts \\

GiGPO
& Step & \xmark & \xmark & \checkmark & Step relative adv. & Anchor states \\

VinePPO
& Step & \xmark & \xmark & \checkmark & Monte Carlo value & Vine branching \\

Turn-PPO
& Turn & \checkmark & \xmark & \xmark & Turn return & Turn MDP \\

CAPO
& Action & \checkmark & \xmark & \xmark & Action value & Action optimization \\

SORL
& Turn & \checkmark & \checkmark & \xmark & Turn-level adv. & Off-policy turns \\

\textbf{\ourmethod~(ours)}
& \textbf{Step} & \checkmark & \checkmark & \checkmark & \textbf{Return + Success} & \textbf{None} \\

\bottomrule
\end{tabular}
\vspace{-10pt}
\end{table*}

\vspace{-10pt}
\section{Preliminaries And Analysis}
\label{sec-preliminaries}
\vspace{-6pt}

\subsection{Problem Setup}
\label{sec:problem_setup}

We model language agent interaction as a partially observable Markov decision process (POMDP)
$\mathcal{M}=(\mathcal{S},\mathcal{A},\mathcal{O},\mathcal{P},\mathcal{R},\Omega,\gamma)$,
where $\mathcal{S}$, $\mathcal{A}$, and $\mathcal{O}$ are the state, action, and observation spaces; $\mathcal{P}$, $\mathcal{R}$, and $\Omega$ denote the transition kernel, reward function, and observation model; and $\gamma$ is the return discount factor.
At environment step $t$, the agent observes $o_t$ and conditions on the interaction history
$h_t=(o_0,a_0,r_0,\ldots,o_{t-1},a_{t-1},r_{t-1},o_t)$.
It generates action $a_t=(y_{t,1},\ldots,y_{t,L_t})$, where $y_{t,i}$ is token $i$, according to
$\pi_\theta(a_t\mid h_t)=\prod_{i=1}^{L_t}\pi_\theta(y_{t,i}\mid h_t,y_{t,<i})$.
The environment transitions after the complete action is executed. An episode of $T$ steps produces
$\tau=(o_0,a_0,r_0,\ldots,a_{T-1},r_{T-1},o_T)$.

Each benchmark defines a binary success outcome $z\in\{0,1\}$ under its evaluation protocol (Appendix~\ref{app:framework_details}). Unlike cumulative return, which aggregates trajectory rewards, $z$ indicates whether the task is completed. For prefix $h_t$, we define the probability of success as
$p_{\mathrm{succ}}^\pi(h_t)=\Pr_\pi(z=1\mid h_t)$.

Standard actor--critic methods \citep{konda2000actor,mnih2016asynchronous} estimate the expected discounted return
$V^\pi(h_t)
=
\mathbb{E}_{\pi}\!\left[
\sum_{k=t}^{T-1}\gamma^{k-t}r_k
\mid h_t
\right]$, with a learned critic $V_\psi(h_t)$. Let $d_t\in\{0,1\}$ indicate episode termination. Given the TD residual
$\delta_t^V=r_t+\gamma(1-d_t)V_\psi(h_{t+1})-V_\psi(h_t)$,
the conventional return-based advantage is computed using GAE \citep{schulman2015gae}:
{\small
\begin{equation}
\hat A_t^V
=
\sum_{l=0}^{T-t-1}
(\gamma\lambda)^l\delta_{t+l}^V,
\label{eq:return_gae}
\end{equation}
}
where $\lambda\in[0,1]$ is the GAE parameter. We refer to $\hat A_t^V$ as the return-based advantage. Our focus is whether these return-based local scores reliably rank alternative actions, rather than only predicting return accurately in aggregate.

\vspace{-6pt}
\subsection{Analysis}
\label{sec:preliminary_analysis}
\vspace{-6pt}

We first ask whether a PPO critic can provide reliable \emph{local} credit for long-horizon agent decisions. For policy improvement, the critic must do more than predict return: it must rank alternative actions correctly. This ranking is hard to learn from a limited number of rollouts when feedback is sparse, especially for decisions far from the terminal outcome\citep{li2025turnppo,kazemnejad2025vineppo}.

To test this directly, we start from the same interaction prefix, sample different candidate actions, and roll out each action several times. For each prefix $h_t$, we estimate a Monte Carlo state value $\hat V_{\mathrm{MC}}(h_t)=M^{-1}\sum_m G_t^{(m)}$. We further sample $K$ candidate actions and evaluate each with $M$ continuations, yielding $\hat Q_{\mathrm{MC}}^{R}(h_t,a_t^{(j)})=M^{-1}\sum_m G_t^{(j,m)}$. We measure critic estimation with explained variance and local credit quality with pairwise ordering accuracy between the critic score $\hat Q_V(h_t,a_t)=r_t+\gamma(1-d_t)V_\psi(h_{t+1})$ and the Monte Carlo action values. Prefixes are divided into \emph{Near}, \emph{Mid}, and \emph{Far} according to benchmark-specific tertiles of the remaining horizon $H_t=T-t$.

\begin{wrapfigure}{r}{0.55\columnwidth}
    \vspace{-20pt}
    \centering
    \includegraphics[width=\linewidth]{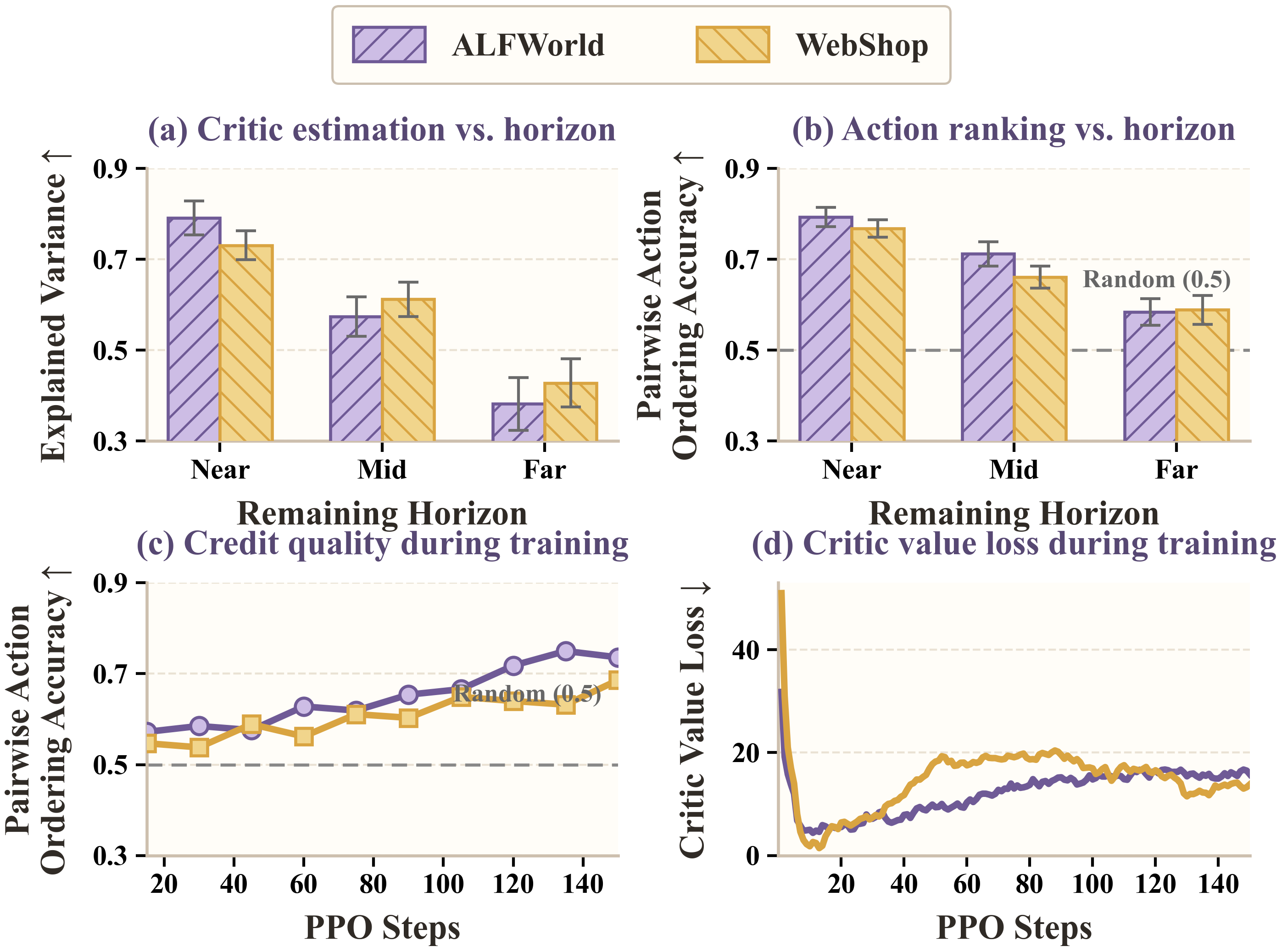}
    \caption{Vanilla PPO critic reliability analysis on ALFWorld and WebShop.}
    \label{fig:ppo_analysis}
    \vspace{-18pt}
\end{wrapfigure}

\textit{Observation 1: Critic estimation degrades with horizon.} Figure~\ref{fig:ppo_analysis}(a) shows a consistent drop in explained variance from near to far horizon prefixes across both benchmarks, indicating that return estimation becomes less reliable as decisions move farther from outcome supervision.

\textit{Observation 2: Local action credit degrades more critically.} Figure~\ref{fig:ppo_analysis}(b) shows that pairwise action ordering accuracy decreases with horizon and approaches the random baseline for far horizon. This result shows that the critic's ordering increasingly diverges from Monte Carlo estimates for decisions farther from terminal outcome. Thus, even a critic with nontrivial return-prediction ability may fail to distinguish better actions, consistent with prior findings on learned value models \citep{kazemnejad2025vineppo}.

\textit{Observation 3: Credit quality changes throughout training.} Figure~\ref{fig:ppo_analysis}(c) shows that action-ordering accuracy improves overall, but it does not improve steadily across updates. Fig.~\ref{fig:ppo_analysis}(d) shows that critic value loss initially decreases but later rebounds and fluctuates, showing that critic quality remains unstable throughout training rather than improving steadily with policy updates.\citep{pan2026evpo}.

Together, these results show a limitation of the standard return critic: its action rankings become less reliable for decisions far from the final outcome and can fluctuate during training. This motivates adding a direct success signal and using it only when the critic is reliable.

\vspace{-8pt}
\section{Methodology}
\label{sec-method}
\vspace{-8pt}

\subsection{Overview}
\label{sec:overview}
\vspace{-4pt}

As illustrated in Fig.~\ref{fig:method_overview}, \ourmethod addresses the gap between return prediction and reliable credit assignment in long trajectories. While return critics can implicitly encode task success, this signal may be unreliable for local credit assignment in practice. \ourmethod therefore augments the return critic with an \textbf{Success Value} head. \textbf{Success-GAE} uses changes in this prediction to compute step credit, and \ourmethod scales this credit according to the \textbf{reliability} of the success critic. The resulting step advantage is finally broadcast to all tokens of the corresponding action for token-level PPO optimization.

\begin{figure*}[t]
    \vspace{-10pt}
    \centering
    \includegraphics[width=\textwidth]{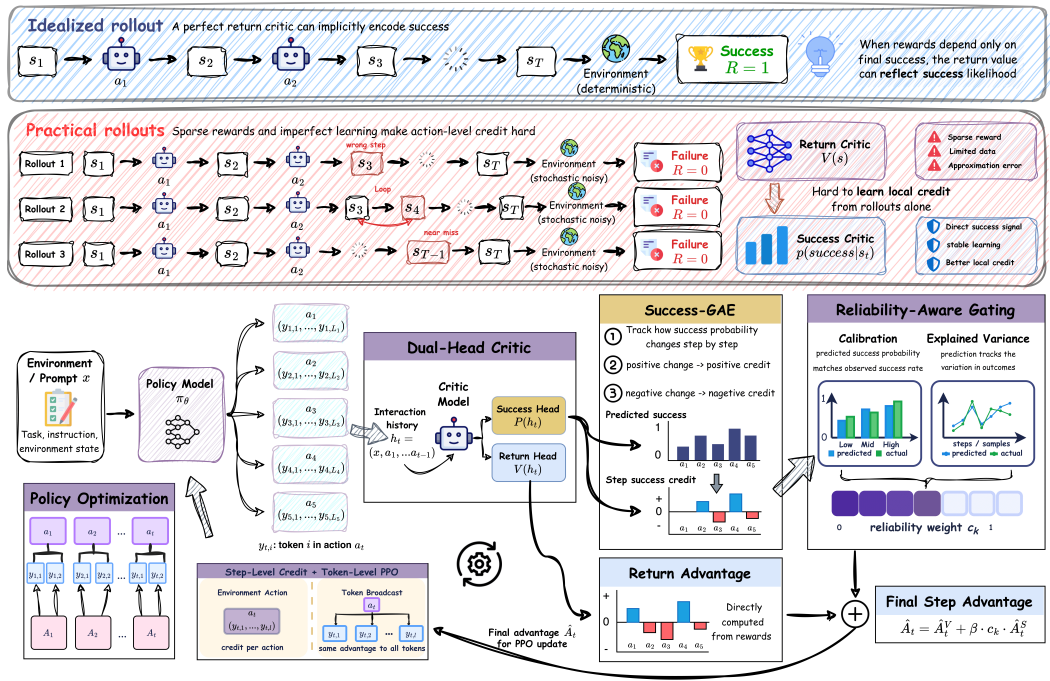}
    \caption{
    \textbf{Overview of \ourmethod.}
    \textbf{Top:} A simple rollout example contrasts an ideal return critic with practical training, where sparse supervision and imperfect learning make reliable action-level credit difficult.
    \textbf{Bottom:} \ourmethod learns success-aware critic signals, converts them into reliable environment-step credit, and uses the resulting advantages for token-level policy optimization.
    }
    \vspace{-15pt}
    \label{fig:method_overview}
\end{figure*}

\subsection{Success-Aware Dual-Head Critic}
\label{sec:dual_head_critic}

The preceding analysis suggests that accurate return estimation alone may not provide reliable local credit for eventual task success. We therefore augment the conventional return critic with an explicit success-prediction head. Given an interaction prefix $h_t$, the dual-head critic outputs $\mathcal{C}(h_t)=\big(V_\psi(h_t),P_\phi(h_t)\big)$, where $V_\psi(h_t)$ estimates future return and $P_\phi(h_t)$ denotes the \textbf{Success Value} defined in Sec.~\ref{sec:problem_setup}. The return head predicts expected future reward. The success head predicts whether the task will eventually be completed from the current prefix. To jointly learn these two signals from collected trajectories, we optimize return regression together with success prediction:
\begin{equation}\small
\mathcal{L}_{\mathrm{critic}}
=
\mathbb{E}_{\tau}
\left[
\frac{1}{T}
\sum_{t=0}^{T-1}
\left(
\big(V_\psi(h_t)-\hat G_t\big)^2
+
\eta_S\,\mathrm{BCE}\!\left(P_\phi(h_t),S\right)
\right)
\right].
\label{eq:critic_objective}
\end{equation}
Here, $\hat G_t$ is the return target for prefix $h_t$, $S$ is the trajectory-level success label, and $\eta_S$ balances return regression and success prediction. Because each interaction prefix is paired with the eventual task outcome, the success head directly learns to predict eventual task completion from that prefix, rather than relying on this information to emerge implicitly from return fitting. This explicit Success Value will serve as the basis for constructing success-aware environment-step credit in the next stage.

\vspace{-4pt}
\subsection{Success-GAE}
\label{sec:success_gae}
\vspace{-4pt}

The Success Value is defined on interaction prefixes, but policy learning requires credit for the action that moves the agent from $h_t$ to $h_{t+1}$. We therefore turn successive success predictions into an action-level signal. Let $p_t=P_\phi(h_t)$ and set the terminal prediction to the observed outcome, $p_T=S$. We define the success temporal difference as $\delta_t^S=p_{t+1}-p_t$: an action receives positive credit when it increases the predicted probability of eventual success and negative credit when it decreases it.

A single temporal difference captures only the immediate change in success prospects, whereas long-horizon actions may influence several subsequent steps. We therefore aggregate future success differences in the same spirit as generalized advantage estimation and define \textbf{Success-GAE} as
\begin{equation}\small
\hat A_t^S
=
\sum_{l=0}^{T-t-1}
\lambda_S^l\delta_{t+l}^S,
\qquad
\lambda_S\in[0,1].
\label{eq:success_gae}
\end{equation}
Here, $\hat A_t^S$ is the success-oriented advantage for action $a_t$, $\delta_{t+l}^S$ measures subsequent changes in predicted success, and $\lambda_S$ controls how strongly later changes contribute to the current action. Unlike return-based GAE, Success-GAE does not use the environment discount factor because it tracks changes in task-success probability rather than accumulated reward. The resulting signal assigns credit based on whether intermediate decisions make eventual success more or less likely.

\vspace{-4pt}
\subsection{Reliability-Aware Credit Integration}
\label{sec:reliability_integration}
\vspace{-4pt}

\paragraph{Reliability Estimation.}
Because Success-GAE is useful only when the Success Value predictor is reliable, we assess its quality at optimization stage $k$ using the Brier score and explained variance:
{\small
\begin{align}
\mathrm{BS}_k
&=
\frac{1}{N_k}
\sum_{n=1}^{N_k}
\left(P_\phi(h^{(n)})-S^{(n)}\right)^2,
\\[-2pt]
\mathrm{EV}_k^S
&=
1-
\frac{\operatorname{Var}\!\left(S-P_\phi(h)\right)}
{\operatorname{Var}(S)+\varepsilon}.
\label{eq:success_reliability_metrics}
\end{align}
}
Here, $\mathrm{BS}_k$ measures absolute probabilistic error, whereas $\mathrm{EV}_k^S$ measures whether prediction differences track variation in eventual task success. The two metrics therefore capture complementary aspects of reliability and provide a more complete view of critic quality.

\vspace{-4pt}
\paragraph{Reliability Weighting.}
We next convert these metrics into a weight in $[0,1]$. Let $\bar S_k=N_k^{-1}\sum_n S^{(n)}$ denote the success rate in the current batch. A constant predictor $P_\phi(h)=\bar S_k$ has Brier score $\bar S_k(1-\bar S_k)$, so we use this error as a natural reference for prediction quality and clip explained variance to obtain a comparable reliability scale:
{\small
\begin{align}
q_k^{\mathrm{BS}}
&=
\operatorname{clip}\!\left(
1-
\frac{\mathrm{BS}_k}
{\bar S_k(1-\bar S_k)+\varepsilon},
0,1
\right),
\\[-2pt]
q_k^{\mathrm{EV}}
&=
\operatorname{clip}\!\left(
\mathrm{EV}_k^S,
0,1
\right),
\qquad
c_k=q_k^{\mathrm{BS}}q_k^{\mathrm{EV}}.
\label{eq:success_reliability}
\end{align}
}
Thus, $q_k^{\mathrm{BS}}$ is positive only when the predictor outperforms the constant baseline, while $q_k^{\mathrm{EV}}$ retains only positive explained variance. Their product $c_k$ gives substantial weight to success credit only when both criteria support reliable and informative predictions.

\vspace{-4pt}
\paragraph{Credit Integration.}
Finally, we combine the return and weighted success advantages:
\begin{equation}\small
\hat A_t
=
\hat A_t^V
+
\beta\,c_k\,\hat A_t^S,
\label{eq:final_advantage}
\end{equation}
where $\beta$ controls the contribution of success-based credit. When $c_k$ is small, the update falls back to the return advantage; as reliability improves, success-based credit contributes more.

\vspace{-4pt}
\subsection{Token-Level Policy Optimization}
\label{sec:policy_optimization}
\vspace{-4pt}

To optimize the autoregressive policy, we assign this advantage to all tokens in the action and retain the standard token-level clipped PPO objective \citep{schulman2017ppo}. We optimize:
\begin{equation}\small
\mathcal{L}_{\mathrm{actor}}(\theta)
=
-\mathbb{E}\!\left[
\frac{1}{\sum_t L_t}
\sum_t\sum_{i=1}^{L_t}
\min\!\left(
\rho_{t,i}(\theta)\hat A_t,\,
\operatorname{clip}\!\left(
\rho_{t,i}(\theta),1-\epsilon,1+\epsilon
\right)\hat A_t
\right)
\right].
\label{eq:actor_objective}
\end{equation}
All tokens within the same action share the environment-step credit $\hat A_t$, while the policy update remains standard token-level PPO. Thus, \ourmethod augments the action advantage with reliability-weighted success credit while leaving the PPO actor objective unchanged. Appendix~\ref{app:theory} provides a theoretical analysis of when the added success signal can improve the distinction between actions.

\vspace{-4pt}
\section{Experiment}
\label{sec:experiments}
\vspace{-4pt}

\begin{table*}[t]
\centering
\caption{
\textbf{Main results across agentic benchmarks.}
We report success rate (\%) on ALFWorld, and task score and success rate on ScienceWorld and WebShop. Results are averaged over three runs. \protect\colorbox{bestbg}{\textbf{Best}} and \protect\colorbox{secondbg}{\underline{second-best}} are highlighted within each base-model block.
}
\label{tab:main_results}

\resizebox{\textwidth}{!}{
\begin{tabular}{lccccccccccc}
\toprule

\multirow{2}{*}{\textbf{Method}}
& \multicolumn{7}{c}{\textbf{ALFWorld}}
& \multicolumn{2}{c}{\textbf{ScienceWorld}}
& \multicolumn{2}{c}{\textbf{WebShop}} \\

\cmidrule(lr){2-8}
\cmidrule(lr){9-10}
\cmidrule(lr){11-12}

\rowcolor{subheaderbg}
& \color{subheadertext}\textbf{Pick}
& \color{subheadertext}\textbf{Look}
& \color{subheadertext}\textbf{Clean}
& \color{subheadertext}\textbf{Heat}
& \color{subheadertext}\textbf{Cool}
& \color{subheadertext}\textbf{Pick2}
& \color{subheadertext}\textbf{All}
& \color{subheadertext}\textbf{Score}
& \color{subheadertext}\textbf{Succ}
& \color{subheadertext}\textbf{Score}
& \color{subheadertext}\textbf{Succ} \\

\midrule


\rowcolor{sectionbg}
\multicolumn{12}{l}{\textbf{\textit{Closed-Source Models}}} \\

GPT-4o
& 75.3 & 60.8 & 31.2 & 56.7 & 21.6 & 49.8 & 48.0
& 54.3 & 45.4 & 31.8 & 23.7 \\

Gemini-2.5-Pro
& 92.8 & 63.3 & 62.1 & 69.0 & 26.6 & 58.7 & 60.3
& 47.8 & 36.7 & 42.5 & 35.9 \\

\midrule


\rowcolor{sectionbg}
\multicolumn{12}{l}{
\textbf{\textit{Base: Qwen2.5-1.5B-Instruct}}
} \\

Direct Prompt
& 5.9 & 5.5 & 3.3 & 9.7 & 4.2 & 0.0 & 4.1
& 5.9 & 0.7 & 23.1 & 5.2 \\

ReAct
& 17.4 & 20.5 & 15.7 & 6.2 & 7.7 & 2.0 & 12.8
& 9.0 & 1.2 & 40.1 & 11.3 \\

Reflexion
& 35.3 & 22.2 & 21.7 & 13.6 & 19.4 & 3.7 & 21.8
& 7.1 & 3.9 & 55.8 & 21.9 \\

PPO
& 64.8 & 40.5 & 57.1 & 60.6 & 46.4 & 47.4 & 54.4
& 29.3 & 10.9 & 73.8 & 51.5 \\

RLOO
& 88.3 & 52.8 & 71.0 & 62.8 & 66.4 & 56.9 & 69.7
& 30.6 & 18.4 & 73.9 & 52.1 \\

GRPO
& 85.3 & 53.7 & 84.5 & 78.2 & 59.7 & 53.5 & 72.8
& 31.7 & 21.1 & 75.8 & 56.8 \\

GSPO
& 91.6 & 68.4 & 87.2 & 84.6 & 65.8 & 68.7 & 79.3
& 33.9 & 23.6 & 79.7 & 60.9 \\

GiGPO
& 96.0 & 76.5 & 91.8 & 91.3 & 71.7 & 79.5 & 86.1
& 35.6 & 25.8 & 83.1 & 65.0 \\

BEACON
& \best{100} & \second{88.2} & 86.7 & \best{100} & \second{78.9} & \best{92.9} & \second{91.4}
& \best{58.9} & \second{45.3} & \second{86.1} & 75.6 \\

ST-PPO
& \second{97.4} & 79.1 & \second{93.6} & 89.7 & 75.4 & 83.1 & 87.7
& 40.8 & 31.3 & 84.8 & 68.0 \\

CAPO
& \best{100} & 84.5 & \best{94.1} & \second{96.2} & 77.6 & \second{89.4} & 91.2
& 52.6 & 40.6 & \best{86.4} & \second{76.6} \\

\rowcolor{oursbg}
\textbf{\ourmethod (Ours)}
& \best{100} & \best{97.4} & 88.9 & \best{100} & \best{93.3} & 82.6 & \best{93.4}
& \second{57.8} & \best{46.8} & 85.3 & \best{78.1} \\

\midrule


\rowcolor{sectionbg}
\multicolumn{12}{l}{
\textbf{\textit{Base: Qwen2.5-7B-Instruct}}
} \\

Direct Prompt
& 33.4 & 21.6 & 19.3 & 6.9 & 2.8 & 3.2 & 14.8
& 11.4 & 4.2 & 26.4 & 7.8 \\

ReAct
& 48.5 & 35.4 & 34.3 & 13.2 & 18.2 & 17.6 & 31.2
& 17.4 & 7.8 & 46.2 & 19.5 \\

Reflexion
& 62.0 & 41.6 & 44.9 & 30.9 & 36.3 & 23.8 & 42.7
& 23.4 & 11.7 & 58.1 & 28.8 \\

PPO
& 92.3 & 64.0 & 92.5 & 89.5 & 80.3 & 68.8 & 80.4
& 37.1 & 24.0 & 81.4 & 68.7 \\

RLOO
& 87.6 & 78.2 & 87.3 & 81.3 & 71.9 & 48.9 & 75.5
& 57.4 & 44.6 & 80.3 & 65.7 \\

GRPO
& 90.8 & 66.1 & 89.3 & 74.7 & 72.5 & 64.7 & 77.6
& 61.8 & 49.1 & 79.3 & 66.1 \\

GSPO
& 94.6 & 74.8 & 93.7 & 79.6 & 80.1 & 72.5 & 84.5
& 65.2 & 51.0 & 81.7 & 69.3 \\

GiGPO
& \second{97.7} & 82.7 & \second{98.8} & 83.7 & 89.3 & 79.2 & 90.8
& 69.2 & 53.4 & 84.4 & 72.8 \\

BEACON
& \best{100} & 81.8 & 96.3 & \best{92.9} & \second{94.7} & \second{90.0} & \second{94.5}
& \second{83.7} & \second{64.3} & 87.7 & \second{79.7} \\

ST-PPO
& \best{100} & \second{86.4} & 97.5 & 88.2 & 91.0 & 84.6 & 92.7
& 73.8 & 56.9 & 86.2 & 75.4 \\

CAPO
& \best{100} & 84.7 & \best{100} & \second{91.6} & 93.6 & 88.4 & \second{94.5}
& 79.8 & 61.8 & \second{88.9} & 78.0 \\

\rowcolor{oursbg}
\textbf{\ourmethod (Ours)}
& \best{100} & \best{100} & \best{100} & 91.3 & \best{100} & \best{90.5} & \best{97.4}
& \best{86.7} & \best{66.5} & \best{92.7} & \best{83.8} \\

\bottomrule
\end{tabular}
}
\vspace{-15pt}
\end{table*}

\subsection{Experimental Setup}
\label{sec:experimental_setup}

\paragraph{Benchmarks.}

We evaluate \ourmethod on five long-horizon agentic benchmarks spanning two major interaction paradigms. ALFWorld \citep{shridhar2021alfworld}, WebShop \citep{yao2022webshop}, and ScienceWorld \citep{wang2022scienceworld} represent text-based interactive environments, covering embodied household tasks, web navigation, and scientific reasoning. BFCL-v4 \citep{patil2025bfcl} and AppWorld \citep{trivedi2024appworld} evaluate tool-using agents through function calling and application-level execution. These benchmarks let us test \ourmethod across different interaction structures, action spaces, and task horizons. Full benchmark details are provided in Appendix~\ref{app:framework_details}.

\vspace{-4pt}
\paragraph{Baselines.}

We compare \ourmethod against representative prompting and RL baselines. Closed-source references include GPT-4o \citep{openai2024gpt4o} and Gemini-2.5-Pro \citep{comanici2025gemini}. Prompting baselines include Direct Prompt, ReAct \citep{yao2023react}, and Reflexion \citep{shinn2023reflexion}. RL baselines cover three categories: (1) Standard policy optimization, including PPO \citep{schulman2017ppo}, RLOO \citep{ahmadian2024rloo}, GRPO \citep{shao2024deepseekmath}, and GSPO \citep{zheng2025gspo}; (2) Fine-grained or structured credit assignment, including GiGPO \citep{feng2025gigpo}, BEACON \citep{wang2026beacon}, ST-PPO \citep{li2025stppo}, and CAPO \citep{wang2026capo}; and (3) Critic-free sequence-level optimization, represented by SeeUPO \citep{hu2026seeupo}. We use official implementations or closely follow the reported configurations whenever available.

\vspace{-4pt}
\paragraph{Implementation.}

Our experiments use Qwen2.5-1.5B-Instruct and Qwen2.5-7B-Instruct \citep{qwen2024qwen25}, together with Qwen3-8B \citep{yang2025qwen3}, as backbone models. ALFWorld, WebShop, and ScienceWorld are implemented with verl-agent \citep{feng2025gigpo}, while BFCL-v4 and AppWorld are built on AgentEvolver \citep{zhai2025agentevolver}. We retain Qwen2.5 due to more stable agentic performance and infrastructure support, and use the same base model across RL methods for fair comparison. Within each backbone and benchmark setting, all trainable methods share the same initial model, rollout budget, environment interface, and evaluation protocol. Detailed configurations and hyperparameters are provided in Appendix~\ref{app:framework_details}.

\vspace{-8pt}
\subsection{Main Results}
\label{sec:main_results}
\vspace{-4pt}

\paragraph{Text-Based Interactive Benchmarks.}
Table~\ref{tab:main_results} reports results on ALFWorld, ScienceWorld, and WebShop. With Qwen2.5-1.5B-Instruct, \ourmethod reaches success rates of 93.4\% on ALFWorld, 46.8\% on ScienceWorld, and 78.1\% on WebShop, exceeding the strongest RL baselines by 2.0, 1.5, and 1.5 percentage points, respectively. With Qwen2.5-7B-Instruct, \ourmethod further improves to 97.4\%, 66.5\%, and 83.8\%, with gains of 2.9, 2.2, and 4.1 percentage points over the strongest baselines. These gains appear across all three text-interaction benchmarks and both model sizes, indicating that the benefit of \ourmethod is not limited to a single environment or backbone scale.

\begin{wraptable}{r}{0.60\columnwidth}
    \vspace{-15pt}
    \centering
    \caption{
    \textbf{Results on agent tool-use benchmarks.}
    We report \textit{avg@4} and \textit{pass@4}.
    \protect\colorbox{bestbg}{\textbf{Best}} and
    \protect\colorbox{secondbg}{\underline{second-best}}
    are highlighted within each base-model block.
    }
    \label{tab:tool_results}

    \scriptsize
    \setlength{\tabcolsep}{2.2pt}
    \renewcommand{\arraystretch}{1.02}

    \resizebox{\linewidth}{!}{
    \begin{tabular}{lcccccc}
    \toprule
    \textbf{Method}
    & \multicolumn{2}{c}{\textbf{AppWorld}}
    & \multicolumn{2}{c}{\textbf{BFCL-v4}}
    & \multicolumn{2}{c}{\textbf{Avg.}} \\
    \cmidrule(lr){2-3}
    \cmidrule(lr){4-5}
    \cmidrule(lr){6-7}

    \rowcolor{subheaderbg}
    & \cellcolor{subheaderbg}\color{subheadertext}\textbf{avg@4}
    & \cellcolor{subheaderbg}\color{subheadertext}\textbf{pass@4}
    & \cellcolor{subheaderbg}\color{subheadertext}\textbf{avg@4}
    & \cellcolor{subheaderbg}\color{subheadertext}\textbf{pass@4}
    & \cellcolor{subheaderbg}\color{subheadertext}\textbf{avg@4}
    & \cellcolor{subheaderbg}\color{subheadertext}\textbf{pass@4} \\
    \midrule

    \multicolumn{7}{l}{\textit{Qwen2.5-7B-Instruct}} \\

    Base
    & 12.63 & 24.56
    & 18.00 & 27.00
    & 15.32 & 25.78 \\

    +PPO
    & 27.36 & 42.11
    & 34.75 & 43.00
    & 31.06 & 42.56 \\

    +GRPO
    & 30.70 & 45.61
    & 37.50 & 46.00
    & 34.10 & 45.81 \\

    +GSPO
    & 25.00 & 40.35
    & 40.25 & 47.00
    & 32.63 & 43.68 \\

    +GiGPO
    & 35.44 & 52.63
    & \second{43.75} & 49.00
    & 39.60 & 50.82 \\

    +SeeUPO
    & \second{40.74} & \best{57.89}
    & 42.50 & \second{50.00}
    & \second{41.62} & \second{53.95} \\

    \rowcolor{gray!15}
    +\ourmethod
    & \best{42.89} & \second{57.02}
    & \best{44.25} & \best{55.00}
    & \best{43.57} & \best{56.01} \\

    \midrule

    \multicolumn{7}{l}{\textit{Qwen3-8B}} \\

    Base
    & 13.63 & 28.07
    & 31.25 & 38.00
    & 22.44 & 33.04 \\

    +PPO
    & 26.68 & 43.86
    & 36.25 & 44.00
    & 31.47 & 43.93 \\

    +GRPO
    & 31.71 & 49.12
    & 37.75 & 47.00
    & 34.73 & 48.06 \\

    +GSPO
    & 25.96 & 43.86
    & 39.25 & 48.00
    & 32.61 & 45.93 \\

    +GiGPO
    & 40.09 & 57.89
    & \second{47.25} & \second{54.00}
    & 43.67 & 55.95 \\

    +SeeUPO
    & \best{50.39} & \second{65.79}
    & 46.25 & \second{54.00}
    & \second{48.32} & \second{59.82} \\

    \rowcolor{gray!15}
    +\ourmethod
    & \second{49.12} & \best{68.42}
    & \best{51.75} & \best{57.00}
    & \best{50.44} & \best{62.47} \\

    \bottomrule
    \end{tabular}
    }

    \vspace{-20pt}
\end{wraptable}

\vspace{-8pt}
\paragraph{Tool-Use Benchmarks.}
Table~\ref{tab:tool_results} evaluates \ourmethod on AppWorld and BFCL-v4 with Qwen2.5-7B-Instruct and Qwen3-8B. \ourmethod achieves the best average performance across the two benchmarks with both backbones, reaching avg@4/pass@4 scores of 43.57/56.01 with Qwen2.5-7B and 50.44/62.47 with Qwen3-8B. On BFCL-v4, \ourmethod obtains 44.25/55.00 and 51.75/57.00, respectively, outperforming all compared RL baselines. On AppWorld, \ourmethod achieves the highest avg@4 with Qwen2.5-7B and the highest pass@4 with Qwen3-8B, while remaining close to SeeUPO on the other metric. Overall, the gains extend from text-interaction benchmarks to tool-use environments.

\vspace{-6pt}
\subsection{Framework Analysis}
\label{sec:credit_quality}
\vspace{-4pt}

\paragraph{Episode Length.}
Figure~\ref{fig:episode_length} compares the average episode length on the three text-interaction benchmarks. \ourmethod reduces the number of environment steps from 27.6 to 19.1 on ALFWorld and from 19.7 to 13.2 on ScienceWorld, while remaining close to GiGPO on WebShop (6.9 vs.\ 6.8). Together with the task results, this suggests that \ourmethod achieves better outcomes with fewer interaction steps, avoiding unnecessarily long action sequences rather than relying on extended exploration.

\vspace{-10pt}
\begin{figure}[!b]
    \centering
    \begin{minipage}[t]{0.48\textwidth}
        \vspace{0pt}
        \centering
        \includegraphics[width=\linewidth]
            {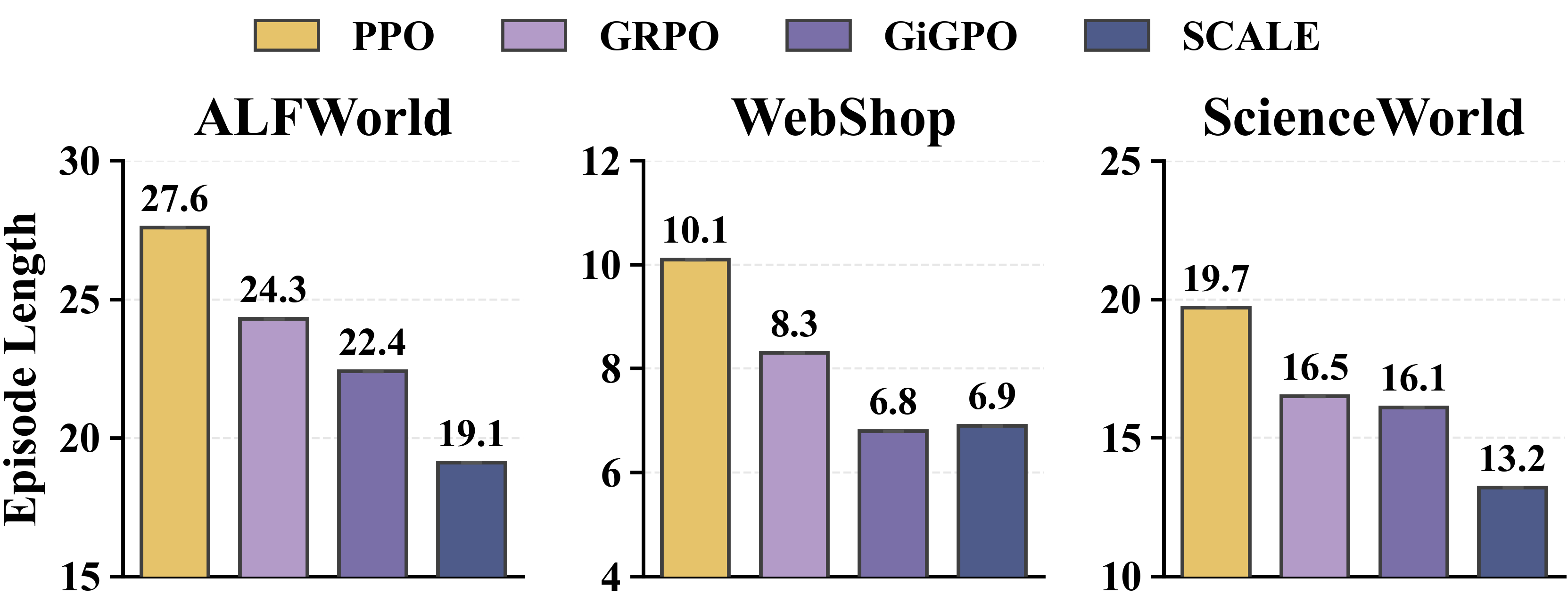}
        \caption{
            \textbf{Episode length on text-interaction benchmarks.}
            We report the average environment steps per
            episode at the final checkpoint on ALFWorld, WebShop,
            and ScienceWorld.
        }
        \label{fig:episode_length}
    \end{minipage}\hfill
    \begin{minipage}[t]{0.48\textwidth}
        \vspace{0pt}
        \centering
        \includegraphics[width=\linewidth]
            {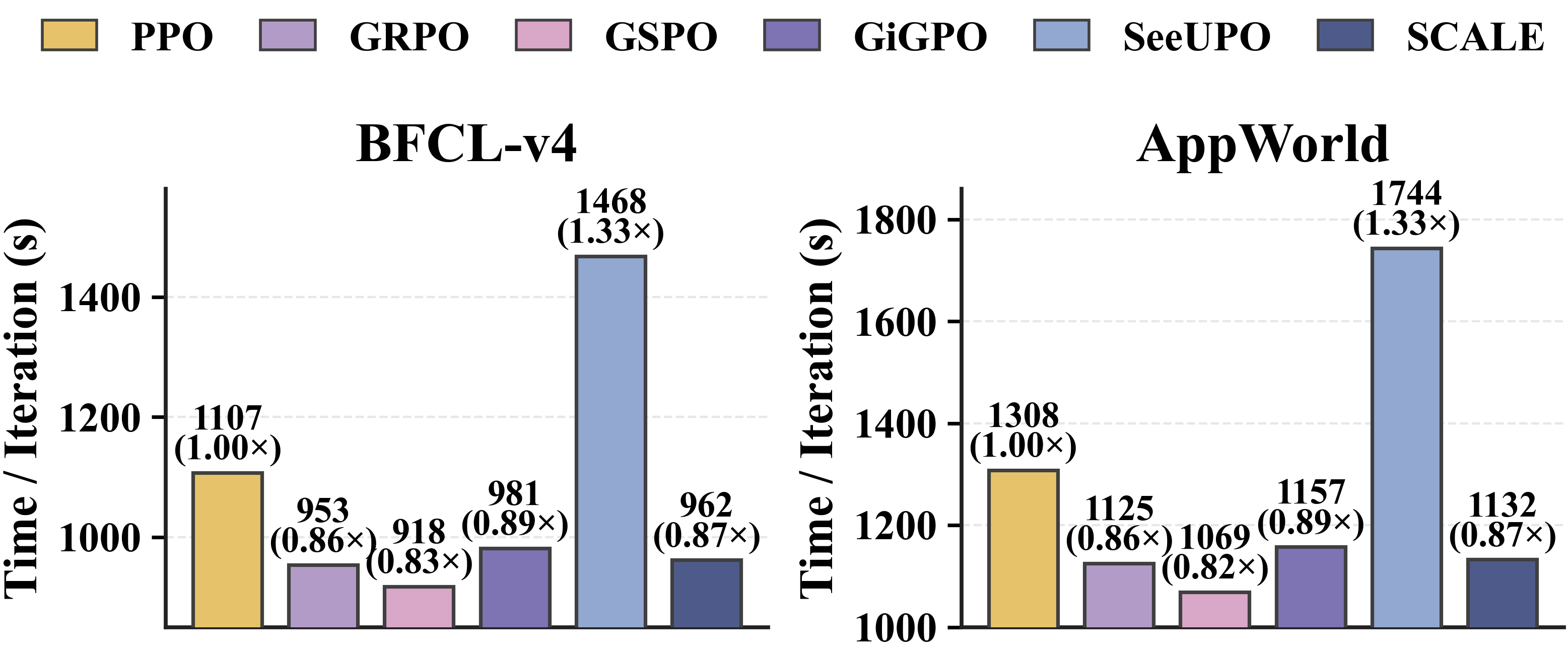}
        \caption{
            \textbf{Training efficiency on BFCL-v4 and AppWorld.}
            We report time per iteration with Qwen3-8B.
            Relative cost is normalized to PPO.
        }
        \label{fig:training_efficiency}
    \end{minipage}
\end{figure}

\paragraph{Training Efficiency.}
Figure~\ref{fig:training_efficiency} compares wall-clock time per training iteration on BFCL-v4 and AppWorld. Although \ourmethod adds a success head and reliability estimation, its training cost remains close to GRPO and GiGPO and is much lower than SeeUPO. This provides the performance gains without a large increase in training cost, suggesting that the additional critic computation introduces only modest practical overhead.

\begin{figure*}[t]
    \centering
    \includegraphics[width=\textwidth]{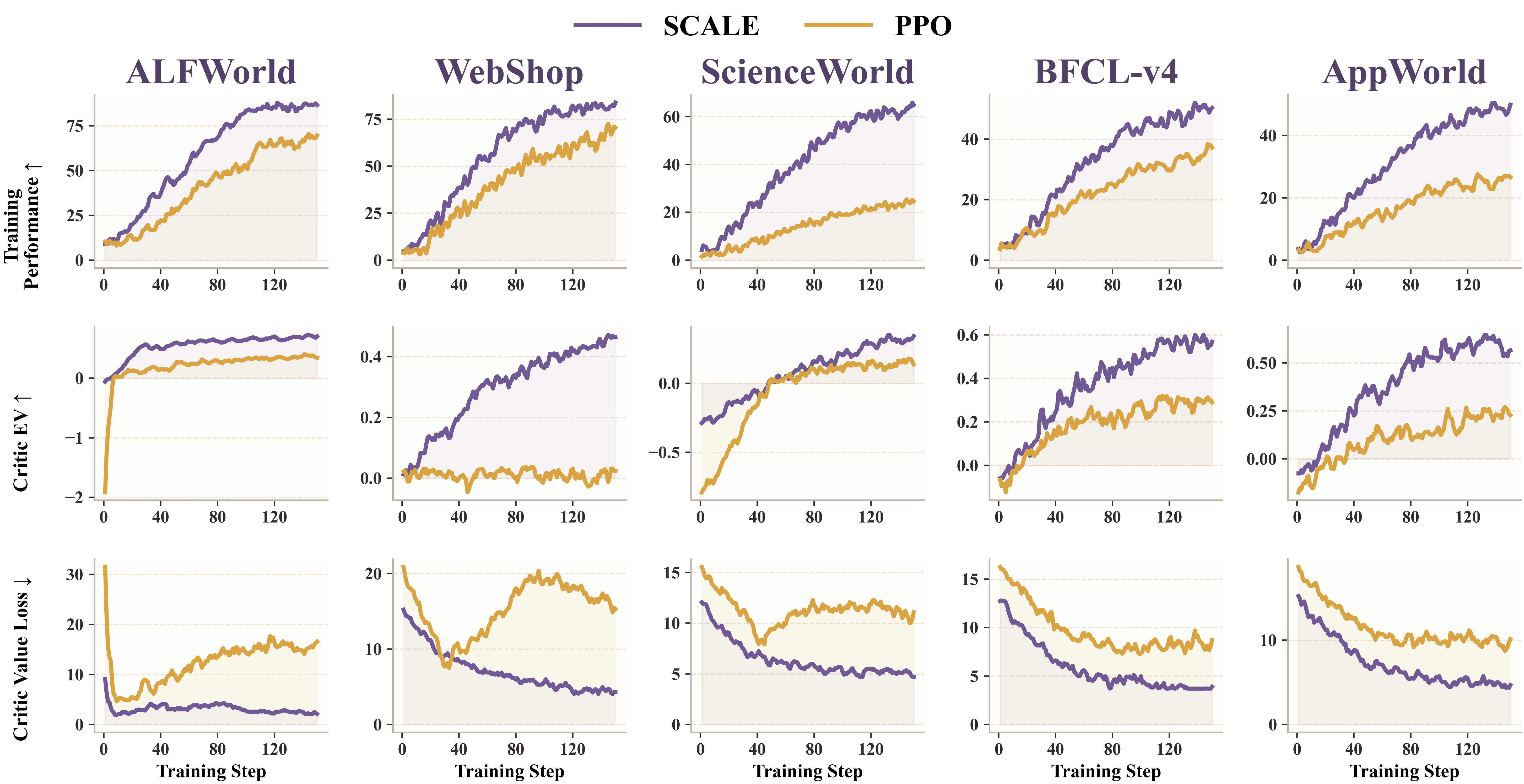}
    \caption{
    \textbf{Training dynamics across five agent benchmarks.}
    \textbf{Top:} Task performance during training, reported as success rate for ALFWorld, WebShop, and ScienceWorld, and \textit{avg@4} for BFCL-v4 and AppWorld.
    \textbf{Middle:} Critic explained variance during training.
    \textbf{Bottom:} Critic value loss during training.
    Compared with PPO, \ourmethod achieves higher explained variance and generally lower, smoother critic loss across the five benchmarks.
    }
    \label{fig:framework_analysis}
    \vspace{-8pt}
\end{figure*}

\vspace{-8pt}
\paragraph{Critic Learning.}
Figure~\ref{fig:framework_analysis} shows that \ourmethod achieves higher critic explained variance than PPO across five benchmarks, indicating that its critic better captures variation in future returns. Its critic loss decreases more smoothly and remains lower, whereas PPO shows larger fluctuations or later increases in several settings. These results suggest that adding the success signal helps the critic learn more stable and informative value estimates, with the same pattern across interaction settings.

\begin{wraptable}{r}{0.56\columnwidth}
\centering
\vspace{-8pt}
\caption{\textbf{Ablation study of \ourmethod with Qwen3-8B.} We report \textit{avg@4} (\%) on BFCL-v4 and AppWorld.}
\label{tab:ablation}
\vspace{-4pt}

\footnotesize
\setlength{\tabcolsep}{4.2pt}
\renewcommand{\arraystretch}{0.96}

\resizebox{\linewidth}{!}{
\begin{tabular}{lcc}
\toprule
\textbf{Variant}
& \textbf{BFCL-v4}
& \textbf{AppWorld} \\
\midrule

\textbf{\ourmethod}
& \textbf{51.75}
& \textbf{49.12} \\

\midrule
\rowcolor{sectionbg}
\multicolumn{3}{l}{\textit{Success-Aware Critic}} \\

\quad w/o Success Value Head
& 43.25 & 39.91 \\

\quad w/o Return Value Head
& 29.75 & 20.18 \\

\midrule
\rowcolor{sectionbg}
\multicolumn{3}{l}{\textit{Success Credit Estimation}} \\

\quad w/ One-Step Success TD
& 39.75 & 31.14 \\

\quad w/o Success Credit
& 46.25 & 43.97 \\

\midrule
\rowcolor{sectionbg}
\multicolumn{3}{l}{\textit{Reliability-Aware Integration}} \\

\quad w/o Reliability Weighting
& 48.00 & 45.83 \\

\quad w/ Brier Score Only
& 49.25 & 47.04 \\

\quad w/ Explained Variance Only
& 50.00 & 47.81 \\

\midrule
PPO
& 36.25 & 26.68 \\

\bottomrule
\end{tabular}
}

\vspace{-10pt}
\end{wraptable}

\vspace{-4pt}
\paragraph{Training Performance.}
The improved critic dynamics are accompanied by stronger policy learning. Across ALFWorld, WebShop, and ScienceWorld, \ourmethod reaches higher training success rates than PPO, while on BFCL-v4 and AppWorld it achieves consistently higher \textit{avg@4}. The performance gap also tends to widen over training, suggesting that better critic learning translates into more effective long-horizon policy optimization rather than merely better value fitting.

\vspace{-4pt}
\paragraph{Ablation Study.}
Table~\ref{tab:ablation} shows that removing any of the tested components lowers the performance of \ourmethod. Removing the Success Value head lowers \textit{avg@4} from 51.75 to 43.25 on BFCL-v4 and from 49.12 to 39.91 on AppWorld, while removing the return head causes larger drops to 29.75 and 20.18. Replacing multi-step success credit with one-step TD also degrades performance to 39.75 and 31.14. Finally, using only Brier score or explained variance remains below the full method on both benchmarks, supporting their combined use for reliability-aware integration. Overall, the results show that return estimation, success-based credit, and reliability weighting each play a distinct role in the full framework.

\vspace{-8pt}
\section{Conclusion}
\vspace{-8pt}

We introduced \ourmethod for credit assignment in long-horizon language agents. Our analysis shows that a critic can retain aggregate return-prediction quality while becoming less reliable at ranking actions, especially for decisions far from terminal supervision. \ourmethod addresses this gap by predicting eventual task success, converting changes in this probability into multi-step action credit, and controlling its contribution according to predictor reliability. Across five agent benchmarks and three model backbones, \ourmethod improves task performance, while maintaining modest training overhead. These results show that improving local credit quality can benefit long-horizon policy learning without changing the underlying token-level PPO objective. Future work could explore incorporating richer signals into critic learning to improve reliability under long-horizon supervision.

\section*{AI use statement} Generative AI tools were used only to improve the readability and linguistic presentation of author-written text. They were not used to develop the methodology, formulate mathematical claims, design experiments, generate or process data. All suggested edits were manually reviewed by the authors, who take full responsibility for the final content of this submission.

\section*{Ethics statement}  This work does not involve human participants, personally identifiable
information, or private and sensitive data. All datasets, software, and experimental procedures were used in accordance with their applicable licenses and research policies. To the best of the authors' knowledge, this work does not raise additional material ethical concerns.      

\section*{Reproducibility statement} We provide the information required to reproduce our results throughout the main paper and supplementary materials. The benchmark datasets are described in
Appendix~\ref{app:framework_details} in detail. Training pipeline includes framework details and evaluation protocols are reported in Appendix~\ref{app:problem-def}.



\bibliography{iclr2027_conference}
\bibliographystyle{iclr2027_conference}


\newpage\clearpage

\appendix

\section{Theoretical Analysis of Success-Aware Credit}
\label{app:theory}

This section analyzes the success-aware credit used in \ourmethod. Our goal is not to prove convergence of token-level PPO. Instead, we study two questions: \textbf{why Success-GAE can provide useful credit for intermediate actions}, and \textbf{when weighting this credit by critic reliability can improve action discrimination}.

\subsection{Local Success-Aware Credit}

Consider an interaction prefix $h$, a continuation policy $\pi$, and $K$ possible actions $\{a_1,\ldots,a_K\}$. After choosing $a_i$, the agent follows $\pi$ until the episode ends and observes $S\in\{0,1\}$. We define
\begin{equation}
\mu_i=\Pr(S=1\mid h,a_i;\pi).
\label{eq:theory_mu}
\end{equation}
Thus, assigning credit at $h$ amounts to estimating and comparing the success probabilities $\{\mu_i\}_{i=1}^K$.

Let $h_0=h$, $\widehat p_j=P_\phi(h_j)$, $\widehat p_H=S$, and $\lambda=\lambda_S$. Success-GAE is
\begin{equation}
\widehat A^S
=
\sum_{j=0}^{H-1}
\lambda^j
(\widehat p_{j+1}-\widehat p_j).
\label{eq:app_success_gae}
\end{equation}
All actions considered at $h$ share the same value $\widehat p_0$. We therefore study
\begin{equation}
\widehat Z=\widehat A^S+\widehat p_0,
\end{equation}
which gives the same ordering over actions as $\widehat A^S$.

\begin{proposition}[Properties of Success-GAE]
\label{prop:success_gae}
The score $\widehat Z$ can be written as
\begin{equation}
\widehat Z
=
(1-\lambda)
\sum_{j=1}^{H-1}
\lambda^{j-1}\widehat p_j
+
\lambda^{H-1}S.
\label{eq:success_weighted}
\end{equation}
Hence $\widehat Z\in[0,1]$.

Let
\begin{equation}
p_j^\star=\Pr(S=1\mid h_j;\pi),
\end{equation}
and let $Z^\star$ be the score in Eq.~\eqref{eq:success_weighted} with $\widehat p_j$ replaced by $p_j^\star$. For every action $a_i$,
\begin{equation}
\mathbb E[Z^\star\mid h,a_i]=\mu_i,
\qquad
\operatorname{Var}(Z^\star\mid h,a_i)
\le
\mu_i(1-\mu_i).
\label{eq:ideal_success_properties}
\end{equation}
If
\begin{equation}
|\widehat p_j-p_j^\star|
\le
\varepsilon_S,
\qquad
1\le j<H,
\end{equation}
then
\begin{equation}
|\widehat Z-Z^\star|
\le
d_S
:=
(1-\lambda^{H-1})\varepsilon_S
\le
\varepsilon_S,
\label{eq:success_error}
\end{equation}
and therefore
\begin{equation}
\left|
\mathbb E[\widehat Z\mid h,a_i]-\mu_i
\right|
\le d_S.
\label{eq:success_bias}
\end{equation}
\end{proposition}

\begin{proof}
We first derive Eq.~\eqref{eq:success_weighted}. Expanding Success-GAE gives
\begin{align}
\widehat A^S
&=
\sum_{j=0}^{H-1}
\lambda^j
(\widehat p_{j+1}-\widehat p_j)
\nonumber\\
&=
-\widehat p_0
+
\sum_{j=1}^{H-1}
(\lambda^{j-1}-\lambda^j)\widehat p_j
+
\lambda^{H-1}S
\nonumber\\
&=
-\widehat p_0
+
(1-\lambda)
\sum_{j=1}^{H-1}
\lambda^{j-1}\widehat p_j
+
\lambda^{H-1}S.
\end{align}
Adding $\widehat p_0$ gives Eq.~\eqref{eq:success_weighted}. For $\lambda\in[0,1]$, all coefficients are nonnegative and
\begin{equation}
(1-\lambda)
\sum_{j=1}^{H-1}\lambda^{j-1}
+
\lambda^{H-1}
=
1.
\end{equation}
Thus, $\widehat Z$ is a convex combination of values in $[0,1]$, so $\widehat Z\in[0,1]$. Since $\widehat p_0$ is the same for all actions considered at $h$, adding it does not change their ordering.

We next consider the case in which the critic predicts the true success probability. Let $\mathcal F_j$ contain the information available at $h_j$, and define
\begin{equation}
M_0=\mu_i,
\qquad
M_j=\mathbb E[S\mid\mathcal F_j]=p_j^\star,
\quad
1\le j\le H.
\end{equation}
At termination, $M_H=S$. By the tower property,
\begin{align}
\mathbb E[M_{j+1}\mid\mathcal F_j]
&=
\mathbb E[
\mathbb E[S\mid\mathcal F_{j+1}]
\mid\mathcal F_j]
\nonumber\\
&=
\mathbb E[S\mid\mathcal F_j]
=
M_j.
\end{align}
Therefore, $(M_j)_{j=0}^{H}$ is a martingale, and
\begin{equation}
\mathbb E[M_j]=\mu_i
\end{equation}
for every $j$. Applying this result to Eq.~\eqref{eq:success_weighted} gives
\begin{align}
\mathbb E[Z^\star]
&=
(1-\lambda)
\sum_{j=1}^{H-1}
\lambda^{j-1}\mathbb E[M_j]
+
\lambda^{H-1}\mathbb E[M_H]
\nonumber\\
&=
\mu_i
\left[
(1-\lambda)
\sum_{j=1}^{H-1}
\lambda^{j-1}
+
\lambda^{H-1}
\right]
=
\mu_i.
\end{align}

We now bound the variance. Define
\begin{equation}
D_j=M_j-M_{j-1},
\qquad
j=1,\ldots,H.
\end{equation}
Since $M_j=M_0+\sum_{k=1}^{j}D_k$, Eq.~\eqref{eq:success_weighted} implies
\begin{equation}
Z^\star-\mu_i
=
\sum_{j=1}^{H}
\lambda^{j-1}D_j.
\label{eq:martingale_success_decomp}
\end{equation}
For $j<k$, $D_j$ is determined by $\mathcal F_{k-1}$, while
\begin{equation}
\mathbb E[D_k\mid\mathcal F_{k-1}]=0.
\end{equation}
Hence
\begin{equation}
\mathbb E[D_jD_k]
=
\mathbb E\!\left[
D_j
\mathbb E[D_k\mid\mathcal F_{k-1}]
\right]
=
0.
\end{equation}
The increments are therefore orthogonal, and
\begin{align}
\operatorname{Var}(Z^\star)
&=
\sum_{j=1}^{H}
\lambda^{2(j-1)}
\mathbb E[D_j^2]
\nonumber\\
&\le
\sum_{j=1}^{H}
\mathbb E[D_j^2]
=
\operatorname{Var}
\left(
\sum_{j=1}^{H}D_j
\right)
\nonumber\\
&=
\operatorname{Var}(S-\mu_i).
\end{align}
Conditioned on $h$ and $a_i$, $S$ is Bernoulli with mean $\mu_i$. Therefore,
\begin{equation}
\operatorname{Var}(Z^\star)
\le
\mu_i(1-\mu_i).
\end{equation}

Finally, we consider an imperfect critic. The terminal value is exact because $\widehat p_H=p_H^\star=S$. Therefore,
\begin{align}
|\widehat Z-Z^\star|
&=
\left|
(1-\lambda)
\sum_{j=1}^{H-1}
\lambda^{j-1}
(\widehat p_j-p_j^\star)
\right|
\nonumber\\
&\le
(1-\lambda)
\sum_{j=1}^{H-1}
\lambda^{j-1}
|\widehat p_j-p_j^\star|
\nonumber\\
&\le
(1-\lambda)
\sum_{j=1}^{H-1}
\lambda^{j-1}\varepsilon_S
\nonumber\\
&=
(1-\lambda^{H-1})\varepsilon_S
=
d_S.
\end{align}
Because $1-\lambda^{H-1}\le1$, we also have $d_S\le\varepsilon_S$. Combining this bound with $\mathbb E[Z^\star]=\mu_i$ gives
\begin{align}
\left|
\mathbb E[\widehat Z]-\mu_i
\right|
&=
\left|
\mathbb E[\widehat Z-Z^\star]
\right|
\nonumber\\
&\le
\mathbb E[
|\widehat Z-Z^\star|
]
\le
d_S.
\end{align}
\end{proof}

As a direct consequence, if $\mu_i-\mu_j>2d_S$, then $\mathbb E[\widehat Z\mid h,a_i]>\mathbb E[\widehat Z\mid h,a_j]$. Thus, Success-GAE preserves the expected ordering of two actions whenever their true success probabilities differ by more than the error allowed by the critic. This connects Proposition~\ref{prop:success_gae} directly to the pairwise ordering metric in Sec.~\ref{sec:preliminary_analysis}.

\subsection{Reliability-Weighted Credit Integration}

We next study what happens when return credit and success credit are combined. Let $Y_i\in[0,1]$ denote the score obtained from the return critic. Suppose
\begin{equation}
\mathbb E[Y_i]=\nu_i+b_i^V,
\qquad
|b_i^V|\le\varepsilon_V,
\qquad
\operatorname{Std}(Y_i)\le\sigma_V,
\label{eq:return_score}
\end{equation}
where $\nu_i$ is the value that $Y_i$ would estimate without critic error.

Let $\alpha\ge0$ specify how much success credit is added. In \ourmethod, this quantity corresponds to the scale set by $\beta c_k$. For analysis, define
\begin{equation}
X_i
=
\frac{Y_i+\alpha\widehat Z_i}{1+\alpha},
\qquad
u_i
=
\frac{\nu_i+\alpha\mu_i}{1+\alpha}.
\label{eq:fused_score}
\end{equation}
The denominator is shared across actions and therefore does not affect their ordering.

Proposition~\ref{prop:success_gae} gives
\begin{equation}
\left|
\mathbb E[X_i]-u_i
\right|
\le
B_\alpha
:=
\frac{\varepsilon_V+\alpha d_S}{1+\alpha}.
\label{eq:fused_bias}
\end{equation}
Suppose the standard deviation of $Z^\star$ is at most $\sigma_S$. We do not require $Y_i$ and $\widehat Z_i$ to be independent. By the triangle inequality for standard deviations,
\begin{equation}
\operatorname{Var}(X_i)
\le
v_\alpha
:=
\min\left\{
\frac14,
\left(
\frac{\sigma_V+\alpha(\sigma_S+d_S)}
{1+\alpha}
\right)^2
\right\}.
\label{eq:fused_variance}
\end{equation}
The bound $B_\alpha$ controls how far the expected score can move because of critic error. The bound $v_\alpha$ controls how much the observed score can vary across rollouts.

\subsection{Finite-Time Analysis of Action Selection}

We use a UCB procedure only to study the consequence of these two bounds. UCB is not part of \ourmethod. Let
\begin{equation}
u_\star=\max_i u_i,
\qquad
\Delta_i=u_\star-u_i.
\end{equation}
For convenience, define
\begin{equation}
g_i=\Delta_i-2B_\alpha.
\label{eq:effective_gap}
\end{equation}
Here, $g_i$ is simply the remaining difference between action $i$ and the best action after allowing for critic bias on both scores.

\begin{theorem}[Bound on local action selection]
\label{thm:local_regret}
Let $N\ge K$, $\delta\in(0,1)$, and
\begin{equation}
L=\log\frac{2KN}{\delta}.
\end{equation}
Suppose that, during the part of training considered here, observations obtained from each action are i.i.d.\ and satisfy Eqs.~\eqref{eq:fused_bias}--\eqref{eq:fused_variance}. If $g_i>0$ for every suboptimal action, then with probability at least $1-\delta$,
\begin{equation}
R_N
\le
\sum_{i:\Delta_i>0}
\Delta_i
\left[
2
+
\frac{32v_\alpha L}{g_i^2}
+
\frac{8L}{3g_i}
\right],
\label{eq:local_regret_bound}
\end{equation}
where
\begin{equation}
R_N
=
\sum_{i:\Delta_i>0}
\Delta_iT_i(N).
\end{equation}
\end{theorem}

\begin{proof}
Bernstein's inequality implies that, with probability at least $1-\delta$, the empirical mean for every action and every sample count up to $N$ lies within
\begin{equation}
r(m)
=
\sqrt{\frac{2v_\alpha L}{m}}
+
\frac{2L}{3m}
\end{equation}
of its expectation.

Suppose a suboptimal action $i$ is selected after it has been observed $m$ times. Its upper confidence bound must then be at least as large as that of the best action. Using Eq.~\eqref{eq:fused_bias} for both actions gives
\begin{equation}
\Delta_i
\le
2B_\alpha+2r(m),
\end{equation}
and therefore
\begin{equation}
g_i\le2r(m).
\end{equation}
However, if
\begin{equation}
m>
\max\left\{
\frac{32v_\alpha L}{g_i^2},
\frac{8L}{3g_i}
\right\},
\end{equation}
then each term in $r(m)$ is smaller than $g_i/4$, which gives $2r(m)<g_i$. This contradicts the condition required for action $i$ to be selected. Hence
\begin{equation}
T_i(N)
\le
2
+
\frac{32v_\alpha L}{g_i^2}
+
\frac{8L}{3g_i}.
\end{equation}
Multiplying by $\Delta_i$ and summing over all suboptimal actions gives Eq.~\eqref{eq:local_regret_bound}.
\end{proof}

The theorem shows why critic bias matters for action discrimination. If $\Delta_i\le2B_\alpha$, the critic error can be large enough to reverse the expected ordering of the two actions, even if more rollouts are collected.

\paragraph{When does success-aware credit help?}
To compare the combined score with return credit alone, consider the case in which both signals target success:
\begin{equation}
\nu_i=\mu_i
\qquad
\text{for all }i.
\end{equation}
If
\begin{equation}
d_S<\varepsilon_V,
\qquad
\sigma_S+d_S<\sigma_V,
\label{eq:success_better}
\end{equation}
then for every $\alpha>0$,
\begin{equation}
B_\alpha<\varepsilon_V
\end{equation}
and
\begin{equation}
\Delta_i-2B_\alpha
>
\Delta_i-2\varepsilon_V.
\end{equation}
The variance bound in Eq.~\eqref{eq:fused_variance} also does not increase. Thus, when the success signal has smaller error and lower variation than the return signal, adding it leaves a larger separation between the expected scores of different actions and tightens the bound in Theorem~\ref{thm:local_regret}.

\subsection{Connection to Reliability Metrics}
\label{app:reliability_metrics}

We finally relate the reliability statistics in \ourmethod to the error terms above. Let
\begin{equation}
p^\star(h)=\mathbb E[S\mid h],
\end{equation}
where $h$ is drawn from the distribution $\mathcal D$ of prefixes visited by the current policy. The population Brier score satisfies
\begin{equation}
\mathrm{BS}
=
\underbrace{
\mathbb E_{\mathcal D}
\left[
(P_\phi(h)-p^\star(h))^2
\right]
}_{\mathcal E_S}
+
\underbrace{
\mathbb E_{\mathcal D}
\left[
p^\star(h)(1-p^\star(h))
\right]
}_{\text{outcome uncertainty}}.
\label{eq:brier_decomp}
\end{equation}
Here, $\mathcal E_S$ is the mean-squared error in the predicted success probability. Since both terms are nonnegative,
\begin{equation}
\mathcal E_S\le\mathrm{BS}.
\end{equation}
A lower Brier score therefore gives a smaller upper bound on this prediction error.

We next relate this average error to prefixes reached after a particular action. Let $\mathcal D_{i,j}$ be the distribution of $h_j$ after choosing action $a_i$. Suppose
\begin{equation}
\frac{d\mathcal D_{i,j}}{d\mathcal D}
\le C.
\label{eq:coverage_condition}
\end{equation}
Then
\begin{align}
\mathbb E_{\mathcal D_{i,j}}
\left[
|P_\phi-p^\star|
\right]
&\le
\sqrt{
\mathbb E_{\mathcal D_{i,j}}
\left[
(P_\phi-p^\star)^2
\right]
}
\nonumber\\
&\le
\sqrt{C\mathcal E_S}
\le
\sqrt{C\,\mathrm{BS}}.
\label{eq:coverage_error}
\end{align}
Using the weights in Eq.~\eqref{eq:success_weighted}, we obtain
\begin{equation}
\left|
\mathbb E[\widehat Z_i]-\mu_i
\right|
\le
(1-\lambda^{H-1})
\sqrt{C\,\mathrm{BS}}.
\label{eq:brier_credit_bias}
\end{equation}
Thus, when $\mathcal D$ covers the prefixes reached after each candidate action, the Brier score also bounds the error in the expected success score.

Explained variance describes a different part of critic quality. It is defined as
\begin{equation}
\mathrm{EV}^S
=
1-
\frac{
\operatorname{Var}(S-P_\phi(h))
}{
\operatorname{Var}(S)+\varepsilon
}.
\label{eq:success_ev_theory}
\end{equation}
If
\begin{equation}
m=\mathbb E[S-P_\phi(h)],
\end{equation}
then
\begin{equation}
\operatorname{Var}(S-P_\phi(h))
=
\mathrm{BS}-m^2.
\end{equation}
The two statistics are therefore related, but they describe the error from different perspectives. Brier score measures squared prediction error, whereas explained variance measures how much residual variation remains relative to the variation in $S$.

For this reason, we use both quantities when computing $c_k=q_k^{\mathrm{BS}}q_k^{\mathrm{EV}}$. We do not claim that this product is an optimal weighting rule. Its role is to reduce the contribution of the success signal when the critic performs poorly according to either statistic. With $\alpha=\beta c_k$, the part of the bias contributed by the success score is at most
\begin{equation}
\frac{\beta c_k}{1+\beta c_k}d_S,
\end{equation}
which approaches zero as $c_k\rightarrow0$.

\paragraph{Summary.}
The analysis gives three results. First, Success-GAE estimates eventual success without changing its expectation when the Success Critic is exact, and its error remains bounded when the critic is approximate. Second, this error determines when the expected ordering of two actions is preserved. Third, if the success signal is more accurate and less variable than the return signal, adding it can improve the separation between actions and tighten the bound on local action selection.

\section{Details of the Benchmarks}\label{app:problem-def}

We evaluate \ourmethod on five benchmarks spanning household interaction, web navigation, scientific experimentation, function calling, and application-level tool use.

\textbf{ALFWorld.}
ALFWorld \citep{shridhar2021alfworld} is a text-based environment derived from ALFRED \citep{shridhar2020alfred} and TextWorld \citep{cote2018textworld}. An agent receives natural-language observations and must complete household tasks by navigating rooms and manipulating objects through textual commands. The benchmark includes tasks such as picking up, cleaning, heating, cooling, and examining objects. We use the standard benchmark split and report task success rate.

\textbf{WebShop.}
WebShop \citep{yao2022webshop} simulates online shopping over a large collection of real products. Given a natural-language instruction, an agent searches for products, navigates product pages, selects options, and attempts to purchase an item that satisfies the requested attributes. The environment provides a graded task score based on how well the selected product matches the instruction, together with a binary success signal. We report both task score and success rate.

\textbf{ScienceWorld.}
ScienceWorld \citep{wang2022scienceworld} contains interactive science tasks in which an agent performs experiments and manipulates objects through natural-language actions. The tasks cover topics from physics, chemistry, and biology and require several dependent steps before the goal can be completed. We evaluate across the benchmark tasks and report both task score and success rate.

\textbf{BFCL-v4.}
BFCL-v4 \citep{patil2025bfcl} evaluates the ability of language models to use external functions correctly. The agent must interpret a user request, choose appropriate functions, provide valid arguments, and use function outputs when further actions are required. This setting tests whether a policy can make reliable tool-use decisions over multiple interaction steps. We follow the benchmark evaluation protocol and report \textit{avg@4} together with \textit{pass@4}.

\textbf{AppWorld.}
AppWorld \citep{trivedi2024appworld} evaluates agents on tasks which interacts with APIs exposed by simulated applications. An agent must determine which applications and APIs are needed, execute sequences of calls, and use the returned information to complete the request. The tasks require dependent operations across applications, making successful completion sensitive to earlier decisions. We follow the standard evaluation protocol and report \textit{avg@4} together with \textit{pass@4}.

\section{Details of the Frameworks}
\label{app:framework_details}

Our experiments cover two forms of agent interaction and therefore use two established training frameworks. For ALFWorld, WebShop, and ScienceWorld, we build on verl-agent \citep{feng2025gigpo} and follow the environment implementation provided by BEACON \citep{wang2026beacon}. For BFCL-v4 and AppWorld, we use the environment and rollout setup released with AgentEvolver \citep{zhai2025agentevolver} and its SeeUPO implementation \citep{hu2026seeupo}. We keep the environment interface, interaction protocol, and evaluation procedure of the corresponding framework, while replacing the policy-optimization component when evaluating different RL methods.

\paragraph{Text-based interaction framework.}
ALFWorld, WebShop, and ScienceWorld are implemented with the interaction pipeline inherited from verl-agent. The framework separates language generation from environment execution: at each environment step, the model receives the task context and the latest environment observation, generates one complete action, and then passes this action to the environment. The next observation is returned only after the action has been executed. This interface matches the environment-step formulation used by \ourmethod, while the policy itself remains autoregressive at the token level. We use the environment implementations from the BEACON codebase for all three benchmarks. BEACON extends \texttt{verl-agent} with support for ScienceWorld while retaining the same interaction structure for ALFWorld and WebShop. The environment manager also records information produced at each transition, such as observations, executed actions, rewards, and episode termination. We use these trajectories to construct the return and success targets required by the two critic heads. No additional intermediate annotation is introduced by \ourmethod.

\paragraph{Interaction context.}
For the three text-based environments, the model input is constructed at every environment step rather than by treating the entire episode as one generated sequence. The input contains the task description, the current observation, a limited record of previous observations and actions, and the actions that are valid in the current state when the environment provides them. This follows the prompt and memory interface of verl-agent. The resulting structure keeps the correspondence between one model response and one environment action explicit.

The three environments use the same general interaction pattern, while adapting the prompt to the task and action space.

\begin{promptbox}[Prompt Template for ALFWorld]
You are an expert agent operating in the ALFRED Embodied Environment. Your task is to:
\textcolor{promptvar}{\{task\_description\}}.

Prior to this step, you have already taken
\textcolor{promptvar}{\{step\_count\}} step(s). Below are the most recent
\textcolor{promptvar}{\{history\_length\}} observations and the corresponding actions you took:
\textcolor{promptvar}{\{action\_history\}}.

You are now at step \textcolor{promptvar}{\{current\_step\}} and your current observation is:
\textcolor{promptvar}{\{current\_observation\}}.

Your admissible actions of the current situation are:
[\textcolor{promptvar}{\{admissible\_actions\}}].

\vspace{3pt}
Now it is your turn to take an action. You should first reason step-by-step about the current situation. This reasoning process MUST be enclosed within
\textcolor{prompttag}{\textless think\textgreater\ldots\textless/think\textgreater}
tags. Once you have finished your reasoning, choose an admissible action for the current step and present it within
\textcolor{prompttag}{\textless action\textgreater\ldots\textless/action\textgreater}
tags.
\end{promptbox}

\begin{promptbox}[Prompt Template for WebShop]
You are an expert autonomous agent operating in the WebShop e-commerce environment.

Your task is to:
\textcolor{promptvar}{\{task\_description\}}.

Prior to this step, you have already taken
\textcolor{promptvar}{\{step\_count\}} step(s). Below are the most recent
\textcolor{promptvar}{\{history\_length\}} observations and the corresponding actions you took:
\textcolor{promptvar}{\{action\_history\}}.

You are now at step \textcolor{promptvar}{\{current\_step\}} and your current observation is:
\textcolor{promptvar}{\{current\_observation\}}.

Your admissible actions for the current situation are:
[\textcolor{promptvar}{\{available\_actions\}}].

\vspace{3pt}
Now it is your turn to take one action for the current step. You should first reason step-by-step about the current situation and decide which admissible action best advances the shopping goal. This reasoning process MUST be enclosed within
\textcolor{prompttag}{\textless think\textgreater\ldots\textless/think\textgreater}
tags. Once you have finished your reasoning, choose an admissible action and present it within
\textcolor{prompttag}{\textless action\textgreater\ldots\textless/action\textgreater}
tags.
\end{promptbox}

\begin{promptbox}[Prompt Template for ScienceWorld]
You are an expert autonomous agent operating in the ScienceWorld environment, a text-based environment centered on tasks from the elementary science curriculum.

Your task is to:
\textcolor{promptvar}{\{task\_description\}}.

Prior to this step, you have already taken
\textcolor{promptvar}{\{step\_count\}} step(s). Below are the most recent
\textcolor{promptvar}{\{history\_length\}} observations and the corresponding actions you took:
\textcolor{promptvar}{\{action\_history\}}.

You are now at step \textcolor{promptvar}{\{current\_step\}} and your current observation is:
\textcolor{promptvar}{\{current\_observation\}}.

Your admissible actions for the current situation are:
[\textcolor{promptvar}{\{available\_actions\}}].

\vspace{3pt}
Now it is your turn to take one action for the current step. You should first reason step-by-step about the current situation and decide which admissible action best advances the task goal. This reasoning process MUST be enclosed within
\textcolor{prompttag}{\textless think\textgreater\ldots\textless/think\textgreater}
tags. Once you have finished your reasoning, choose an admissible action and present it within
\textcolor{prompttag}{\textless action\textgreater\ldots\textless/action\textgreater}
tags.
\end{promptbox}

The environment-specific templates differ mainly in the description of the task. ALFWorld presents household manipulation commands \citep{shridhar2021alfworld}, WebShop presents the shopping goal and currently available page actions \citep{yao2022webshop}, and ScienceWorld presents the scientific task together with the actions available in the current environment state \citep{wang2022scienceworld}. We retain these prompt conventions across all compared RL methods.

\paragraph{Tool-use interaction framework.}
For BFCL-v4 and AppWorld, we follow the training setup in the SeeUPO branch of AgentEvolver \citep{hu2026seeupo}. In this framework, the RL trainer and the benchmark environment communicate through an environment service. The trainer sends the current interaction state and the generated action to this service, which executes the requested function or code and returns the resulting observation. This separation allows the same policy-training pipeline to interact with both benchmarks without embedding their execution logic in the optimizer.

Following the reference setting, BFCL-v4 and AppWorld use asynchronous vLLM rollouts with at most 10 interaction steps, eight rollouts per prompt, and temperature $0.9$. We retain the framework's linear think context template and keep the rollout protocol fixed across backbones.

For BFCL-v4 \citep{patil2025bfcl}, we use the multi-turn data and environment implementation in AgentEvolver. At each turn, the model receives the user request and currently available function definitions, generates a function call, and uses the returned execution result in the next turn. The implementation follows the Qwen function-call format shown below.

\begin{promptbox}[Prompt Template for BFCL-v4]
\textbf{\# Tools}

You may call one or more functions to assist with the user query.

You are provided with function signatures within
\textcolor{prompttag}{\textless tools\textgreater\ldots\textless/tools\textgreater}
tags:

\textcolor{prompttag}{\textless tools\textgreater}

\textcolor{promptvar}{\{function definitions\}}

\textcolor{prompttag}{\textless/tools\textgreater}

\vspace{3pt}
For each function call, return a JSON object containing the function name and arguments within
\textcolor{prompttag}{\textless tool\_call\textgreater\ldots\textless/tool\_call\textgreater}
tags:

\textcolor{prompttag}{\textless tool\_call\textgreater}

\{"name": \textcolor{promptvar}{\{function name\}},
"arguments": \textcolor{promptvar}{\{arguments\}}\}

\textcolor{prompttag}{\textless/tool\_call\textgreater}
\end{promptbox}

where the available functions may change across turns. We preserve the tool definitions and execution results produced by BFCL rather than adding task-specific instructions for \ourmethod. Evaluation uses the BFCL evaluation routines provided by the framework.

For AppWorld \citep{trivedi2024appworld}, we use the official environment and evaluator wrapped by AgentEvolver. The model can inspect available applications and API documentation, execute Python code through the provided APIs, and use each execution result in subsequent steps. The interaction follows the prompt structure below.

\begin{promptbox}[Prompt Template for AppWorld]
You are an AI assistant whose task is to complete the user's request by interacting with applications through their APIs.

You will undertake a multi-step interaction using a Python execution environment. At each step, write Python code, observe the returned output, and use that information to decide the next action.

\vspace{3pt}
To inspect the available applications:

\texttt{print(apis.api\_docs.show\_app\_descriptions())}

\vspace{3pt}
To inspect the APIs provided by an application:

\texttt{print(apis.api\_docs.show\_api\_descriptions(}\\
\texttt{\hspace*{1em}app\_name=\textcolor{promptvar}{\{app name\}}))}

\vspace{3pt}
To inspect the specification of a particular API:

\texttt{print(apis.api\_docs.show\_api\_doc(}\\
\texttt{\hspace*{1em}app\_name=\textcolor{promptvar}{\{app name\}},}\\
\texttt{\hspace*{1em}api\_name=\textcolor{promptvar}{\{api name\}}))}

\vspace{3pt}
Each code execution produces an output that can be used in subsequent steps.

Task:
\textcolor{promptvar}{\{user request\}}
\end{promptbox}

This interaction continues until the task is completed or the maximum number of environment steps is reached. We use the task evaluator supplied by AppWorld to determine completion and obtain the final outcome.

\paragraph{Evaluation protocol.}
We follow the evaluation procedure provided by the corresponding benchmark framework. Table~\ref{tab:evaluation_protocol} summarizes the task score, success criterion, and metric used for each benchmark.

\begin{table*}[t]
\centering
\small
\caption{\textbf{Evaluation protocol across the five benchmarks.} We summarize the task reward or score, the criterion for a successful trajectory, and the reported evaluation metrics.}
\label{tab:evaluation_protocol}
\resizebox{\textwidth}{!}{
\begin{tabular}{lccc}
\toprule
\textbf{Benchmark}
& \textbf{Task Reward / Score}
& \textbf{Success Criterion}
& \textbf{Evaluation Metric} \\
\midrule
ALFWorld
& Environment task completion
& $\texttt{info["won"]}=\texttt{True}$
& Success Rate \\

WebShop
& $q_i\in[0,1]$
& Terminal and $q_i=1$
& Task Score / Success Rate \\

ScienceWorld
& Final score $z_i$
& Terminal and $z_i>0$
& Task Score / Success Rate \\

AppWorld
& $\displaystyle r_i=\frac{P_i}{P_i+F_i}$
& $r_i=1$
& $\mathrm{Avg@}4$ / $\mathrm{Pass@}4$ \\

BFCL-v4 MT
& Binary correctness $(0/1)$
& Official evaluator returns correct
& $\mathrm{Avg@}4$ / $\mathrm{Pass@}4$ \\
\bottomrule
\end{tabular}
}
\end{table*}

For ALFWorld, WebShop, and ScienceWorld, each task is evaluated once using the environment interface inherited from verl-agent. Let $s_i^{(b)}\in\{0,1\}$ indicate whether task $i$ is successful on benchmark $b$. The success rate is
\begin{equation}
\mathrm{SR}^{(b)}
=
\frac{1}{N_b}
\sum_{i=1}^{N_b}
s_i^{(b)}.
\end{equation}
For ALFWorld, success is given directly by the environment completion flag. For WebShop, a trajectory is successful only when the episode terminates with task reward $1$, while the task score is the average continuous environment reward in $[0,1]$. For ScienceWorld, we follow the legacy BEACON\citep{wang2026beacon} criterion, where \texttt{won = isCompleted and score > 0}.
Since \texttt{isCompleted} may also be triggered by the environment step limit, \emph{Succ} is a loose benchmark-specific success metric rather than strict full-task completion. This criterion is applied identically to all methods and baselines, and the Success Critic on ScienceWorld predicts this benchmark-defined binary outcome.

For BFCL-v4 and AppWorld, we follow the SeeUPO setup and sample $K=4$ rollouts for each task. Let $r_{i,j}^{(b)}\in[0,1]$ denote the reward of rollout $j$ for task $i$. We define
\begin{equation}
\mathrm{Avg@}K
=
\frac{1}{N_bK}
\sum_{i=1}^{N_b}
\sum_{j=1}^{K}
r_{i,j}^{(b)}.
\end{equation}
A rollout is counted as fully successful when $r_{i,j}^{(b)}=1$. In AppWorld, the reward is the fraction of passed goal checks, so partial completion contributes to \textit{avg@4} but does not count as full success. In BFCL-v4 Multi-Turn Base, the reward is binary correctness, so \textit{avg@4} is equivalent to the success rate over all sampled rollouts. We additionally define
\begin{equation}
\mathrm{Pass@}K
=
\frac{1}{N_b}
\sum_{i=1}^{N_b}
\mathbf{1}
\left[
\exists j,\;
r_{i,j}^{(b)}=1
\right],
\end{equation}
which measures the fraction of tasks solved at least once among the $K$ sampled rollouts. Unless otherwise stated, the main results report success rate for ALFWorld, task score and success rate for WebShop and ScienceWorld, and \textit{avg@4} for BFCL-v4 and AppWorld.

\paragraph{Consistency across methods.}
Within each benchmark and backbone, all RL methods use the same environment implementation, prompt format, rollout interface, and evaluation procedure. The differences between methods are restricted to their learning objectives and the quantities required by those objectives. In particular, \ourmethod does not introduce additional environment feedback or privileged intermediate supervision. Its Success Critic is trained from the same terminal task outcome that is available in the collected trajectories.

\section{Training Procedure of \ourmethod}
\label{app:algorithm}

This section summarizes the training procedure of \ourmethod. At each iteration, the current policy collects trajectories, the dual-head critic learns return and success prediction from the same rollout data, and the resulting return- and success-based advantages are combined according to Success Critic reliability before the standard token-level PPO update. Algorithm~\ref{alg:scale} gives the complete procedure.

\begin{algorithm}[H]
\caption{Training procedure of \ourmethod.}
\label{alg:scale}
\begin{algorithmic}[1]
\Require Policy $\pi_\theta$, return critic $V_\psi$, Success Critic $P_\phi$
\Require $\gamma,\lambda,\lambda_S,\beta$

\For{training iteration $k=1,2,\ldots$}

    \State Collect trajectories $\mathcal{B}_k$ using $\pi_\theta$
    \State Obtain return targets $\hat G_t$ and task outcomes $S$

    \State Cache pre-update predictions
    $\displaystyle
    v_t^{\mathrm{pre}}=V_\psi(h_t),\quad
    p_t^{\mathrm{pre}}=P_\phi(h_t)
    $
    for all prefixes in $\mathcal{B}_k$

    \State Compute $\mathrm{BS}_k$ and $\mathrm{EV}_k^S$
    from $\{p_t^{\mathrm{pre}}\}$ on $\mathcal{B}_k$
    \State Compute reliability weight
    $\displaystyle c_k=q_k^{\mathrm{BS}}q_k^{\mathrm{EV}}$

    \For{each trajectory $\tau\in\mathcal{B}_k$}
        \State Compute return advantages $\hat A_t^V$
        from $\{v_t^{\mathrm{pre}}\}$ using Eq.~\eqref{eq:return_gae}

        \State Set $p_t=p_t^{\mathrm{pre}}$ and anchor $p_T=S$
        \State Compute $\delta_t^S=p_{t+1}-p_t$
        \State Compute Success-GAE $\hat A_t^S$
        using Eq.~\eqref{eq:success_gae}

        \State Combine credits:
        $\displaystyle
        \hat A_t=\hat A_t^V+\beta c_k\hat A_t^S
        $

        \State Assign $\hat A_t$ to all tokens in action $a_t$
    \EndFor

    \State Update $V_\psi$ and $P_\phi$
    using Eq.~\eqref{eq:critic_objective}

    \State Update $\pi_\theta$
    using Eq.~\eqref{eq:actor_objective}

\EndFor
\end{algorithmic}
\end{algorithm}

Thus, \ourmethod leaves the policy objective unchanged and modifies only how environment-step credit is estimated, calibrated, and integrated before the PPO update. This preserves the standard token-level optimization procedure while allowing the actor to learn from credit signals that more directly reflect eventual task success.

\section{Case Studies of \ourmethod}
\label{app:case_study}

To complement the quantitative results in the main paper, we present two successful trajectories from ScienceWorld and BFCL-v4. These examples show why credit assignment is difficult in long-horizon tasks. An action may have a simple immediate effect, while its importance becomes clear only after later interactions. In ScienceWorld, the agent must complete several task requirements in the correct order. In BFCL-v4, actions in early turns change the filesystem used in later turns. We report the environment observations, actions, scores, user requests, and tool outputs.


\subsection{ScienceWorld: Ordered Subgoals and Delayed Task Completion}
\label{app:case_scienceworld}

We first show a successful \textit{find-living-thing} trajectory from ScienceWorld \citep{wang2022scienceworld}. The task contains several ordered requirements. The agent must first focus on a living thing and then move it to the orange box in the kitchen. As a result, the value of an early action depends on what the agent does in later steps. The trajectory below reports the actions, observations, and native ScienceWorld scores.

\begin{envstep}{Initial State}

\textbf{Task:}

Find a living thing. First, focus on the thing. Then, move it to the orange box in the kitchen.

\end{envstep}


\begin{agentstep}{Step 1}

\casethink{I first need to identify a living thing in the current environment. I do not yet know what objects are available or where the target might be. The task also requires me to focus on the object before moving it, so I should first inspect the surroundings and then choose a suitable target.}

\medskip

\caseaction{look around}

\end{agentstep}

\begin{envstep}{Step 1}

\textbf{Score:} 0.08.

\medskip

This location is called the outside. You see a blue jay egg, a butterfly egg, a dove egg, an axe, and a door to the kitchen.

\end{envstep}


\begin{agentstep}{Step 2}

\casethink{I can now see several eggs, which are possible living things for this task. The blue jay egg is a clear candidate, so I will use it as the target. Since the task explicitly requires focusing on the object before moving it, I should complete that requirement first.}

\medskip

\caseaction{focus on blue jay}

\end{agentstep}

\begin{envstep}{Step 2}

\textbf{Score:} 0.58.

\medskip

You focus on the blue jay egg.

\end{envstep}


\begin{agentstep}{Step 3}

\casethink{The blue jay has now been selected and focused on, so the first requirement is satisfied. The final goal is to move it to the orange box in the kitchen. Before I can travel to the destination, I need to take the object with me, so picking it up is the next necessary step.}

\medskip

\caseaction{pick up blue jay}

\end{agentstep}

\begin{envstep}{Step 3}

\textbf{Score:} 0.67.

\medskip

You move the blue jay to the inventory.

\end{envstep}


\begin{agentstep}{Step 4}

\casethink{The blue jay is now in my inventory, which means I can carry it to the required destination. The task says that the orange box is located in the kitchen, while I am currently outside. I should therefore move to the kitchen and look for the orange box there.}

\medskip

\caseaction{go to kitchen}

\end{agentstep}

\begin{envstep}{Step 4}

\textbf{Score:} 0.83.

\medskip

You move to the kitchen. You see a counter, a table, and an orange box.

\end{envstep}


\begin{agentstep}{Step 5}

\casethink{I have reached the kitchen and can see the orange box. I am already carrying the blue jay, and the earlier focusing requirement has been completed. The only remaining part of the task is to place the target into the orange box, which should complete the task.}

\medskip

\caseaction{move blue jay to orange box}

\end{agentstep}

\begin{envstep}{Outcome}

\textbf{Score:} 1.00.

\medskip

You move the blue jay to the orange box.

\medskip

\textbf{Task Completed.}

\end{envstep}

\paragraph{Credit-assignment perspective.}
The focus action is important because it satisfies a required part of the task before the agent moves the object. Its effect cannot be judged from that step alone. The agent must still pick up the blue jay, move to the kitchen, and place it in the orange box. This example shows why an intermediate action should be evaluated by how it affects the chance of completing the full task. In SCALE, the Success Value estimates this probability from each interaction prefix, and changes in the Success Value provide the signal used by Success-GAE.


\subsection{BFCL-v4: Cross-Turn Credit under Persistent Tool State}
\label{app:case_bfcl}

We next show the official BFCL-v4 Multi-Turn Base instance \texttt{multi\_turn\_base\_0} \citep{patil2025bfcl}, using the ground-truth interaction driver in AgentEvolver. The task contains four user turns, and the filesystem is shared across them. Actions in one turn can therefore affect what is possible in later turns. For example, moving a file in the first turn changes the location used by later search, sorting, and comparison operations. Multiple filesystem calls may be made within each turn.

\begin{envstep}{Initial State}

The initial \texttt{GorillaFileSystem} contains

\medskip

{\ttfamily\small
workspace/document/final\_report.pdf
}

with content

\begin{quote}
\ttfamily\small
Year2024 This is the final report content including budget analysis and other sections.
\end{quote}

and

{\ttfamily\small
workspace/document/previous\_report.pdf
}

with content

\begin{quote}
\ttfamily\small
Year203 This is the previous report content with different budget analysis.
\end{quote}

\end{envstep}


\begin{envstep}{Turn 1}

\textbf{User request:}

Move \texttt{'final\_report.pdf'} within document directory to \texttt{'temp'} directory in document. Make sure to create the directory.

\end{envstep}

\begin{agentstep}{Turn 1}

\toolopen

{\ttfamily\small
\{"name": "cd", "arguments": \{"folder": "document"\}\}
}

\toolclose

\medskip

\toolopen

{\ttfamily\small
\{"name": "mkdir", "arguments": \{"dir\_name": "temp"\}\}
}

\toolclose

\medskip

\toolopen

{\ttfamily\small
\{"name": "mv", "arguments": \{"source": "final\_report.pdf",\\
"destination": "temp"\}\}
}

\toolclose

\end{agentstep}

\begin{envstep}{Turn 1: Tool Results}

\toolopen

{\ttfamily\small
\{"name": "cd\_1", "content":\\
\{"current\_working\_directory": "document"\}\}
}

\toolclose

\medskip

\toolopen

{\ttfamily\small
\{"name": "mkdir\_2", "content": "None"\}
}

\toolclose

\medskip

\toolopen

{\ttfamily\small
\{"name": "mv\_3", "content":\\
\{"result": "'final\_report.pdf' moved to\\
'temp/final\_report.pdf'"\}\}
}

\toolclose

\end{envstep}

\begin{agentstep}{Turn 1: Advance}

{\ttfamily\small
ok.1
}

\end{agentstep}


\begin{envstep}{Turn 2}

\textbf{User request:}

Perform a detailed search using grep to identify sections in the file pertaining to \texttt{'budget analysis'}.

\end{envstep}

\begin{agentstep}{Turn 2}

\toolopen

{\ttfamily\small
\{"name": "cd", "arguments": \{"folder": "temp"\}\}
}

\toolclose

\medskip

\toolopen

{\ttfamily\small
\{"name": "grep", "arguments":\\
\{"file\_name": "final\_report.pdf",\\
"pattern": "budget analysis"\}\}
}

\toolclose

\end{agentstep}

\begin{envstep}{Turn 2: Tool Results}

\toolopen

{\ttfamily\small
\{"name": "cd\_1", "content":\\
\{"current\_working\_directory": "temp"\}\}
}

\toolclose

\medskip

\toolopen

{\ttfamily\small
\{"name": "grep\_2", "content":\\
\{"matching\_lines": [\\
"Year2024 This is the final report content including budget analysis and other sections."\\
]\}\}
}

\toolclose

\end{envstep}

\begin{agentstep}{Turn 2: Advance}

{\ttfamily\small
ok.2
}

\end{agentstep}


\begin{envstep}{Turn 3}

\textbf{User request:}

Upon identifying the requisite \texttt{'budget analysis'} content, sort the \texttt{'final\_report.pdf'} by line for improved clarity and comprehension.

\end{envstep}

\begin{agentstep}{Turn 3}

\toolopen

{\ttfamily\small
\{"name": "sort", "arguments":\\
\{"file\_name": "final\_report.pdf"\}\}
}

\toolclose

\end{agentstep}

\begin{envstep}{Turn 3: Tool Result}

\toolopen

{\ttfamily\small
\{"name": "sort\_1", "content":\\
\{"sorted\_content":\\
"Year2024 This is the final report content including budget analysis and other sections."\}\}
}

\toolclose

\medskip

The BFCL filesystem implementation returns the sorted content; the \texttt{sort} operation does not overwrite the file.

\end{envstep}

\begin{agentstep}{Turn 3: Advance}

{\ttfamily\small
ok.2
}

\end{agentstep}


\begin{envstep}{Turn 4}

\textbf{User request:}

Move \texttt{'previous\_report.pdf'} in document directory to temp as well and having final report also there, proceed to juxtapose it with \texttt{'previous\_report.pdf'} to detect any critical alterations.

\end{envstep}

\begin{agentstep}{Turn 4}

\toolopen

{\ttfamily\small
\{"name": "cd", "arguments": \{"folder": ".."\}\}
}

\toolclose

\medskip

\toolopen

{\ttfamily\small
\{"name": "mv", "arguments":\\
\{"source": "previous\_report.pdf",\\
"destination": "temp"\}\}
}

\toolclose

\medskip

\toolopen

{\ttfamily\small
\{"name": "cd", "arguments": \{"folder": "temp"\}\}
}

\toolclose

\medskip

\toolopen

{\ttfamily\small
\{"name": "diff", "arguments":\\
\{"file\_name1": "final\_report.pdf",\\
"file\_name2": "previous\_report.pdf"\}\}
}

\toolclose

\end{agentstep}

\begin{envstep}{Turn 4: Tool Results}

\toolopen

{\ttfamily\small
\{"name": "cd\_1", "content": \{\}\}
}

\toolclose

\medskip

\toolopen

{\ttfamily\small
\{"name": "mv\_2", "content":\\
\{"result": "'previous\_report.pdf' moved to\\
'temp/previous\_report.pdf'"\}\}
}

\toolclose

\medskip

\toolopen

{\ttfamily\small
\{"name": "cd\_3", "content":\\
\{"current\_working\_directory": "temp"\}\}
}

\toolclose

\medskip

\toolopen

{\ttfamily\small
\{"name": "diff\_4", "content":\\
\{"diff\_lines":\\
"- Year2024 This is the final report content including budget analysis and other sections.\\
+ Year203 This is the previous report content with different budget analysis."\}\}
}

\toolclose

\end{envstep}

\begin{agentstep}{Turn 4: Advance}

{\ttfamily\small
ok.2
}

\end{agentstep}

\begin{envstep}{Outcome}

{\ttfamily\small
[CONVERSATION\_COMPLETED]
}

\medskip

The four per-turn function-call groups correspond to the reference ground-truth interaction for \texttt{multi\_turn\_base\_0}. AgentEvolver subsequently invokes the BFCL sparse evaluator on the completed conversation.

\end{envstep}

\paragraph{Credit-assignment perspective.}
The \texttt{mkdir} and \texttt{mv} calls in the first turn prepare the filesystem for the later requests. After \texttt{final\_report.pdf} is moved to \texttt{temp}, later turns access the same file with \texttt{grep} and \texttt{sort}. The file is also used in the final \texttt{diff} operation. The importance of the first-turn actions therefore becomes clear only after several later turns. This example shows that credit for an early tool call may depend on whether the state it creates supports successful completion of later requests.

\paragraph{Discussion.}
The two examples show two common sources of delayed credit. In ScienceWorld, an early action satisfies one requirement, but the task succeeds only after several later actions. In BFCL-v4, an early tool call changes the filesystem used in later turns. In both cases, the effect of an action cannot be judged only from its immediate result. It also depends on whether the resulting interaction prefix makes eventual task completion more likely. This is the motivation for using Success Value in SCALE and for using its changes to construct Success-GAE.

\end{document}